\documentclass{article} 
\usepackage{arxiv}

\usepackage{amssymb}
\usepackage{amsmath}
\usepackage{graphicx}
\usepackage[mathscr]{euscript}
\usepackage{lineno}
\usepackage{subcaption}
\usepackage[utf8]{inputenc}
\usepackage{float}
\usepackage[normalem]{ulem}
\usepackage{multicol}
\usepackage{lmodern}
\usepackage{lipsum}
\usepackage{marvosym}
\usepackage{mathtools}
\usepackage{soul}
\usepackage{bm}
\usepackage{amsthm}
\usepackage{xargs}                      
\usepackage{xcolor}
\usepackage{bbm}
\usepackage{hyperref}
\usepackage[title]{appendix}
\usepackage{mathtools}

\newcommand{\argmin}{\operatornamewithlimits{argmin}}

\newcommand{\Rea}{\mathbb{R}}

\newcommand{\R}{\mathbb{R}}

\newcommand{\set}[1]{\left \{ #1\right \}}
\newcommand{\norm}[1]{\lVert#1\rVert}
\newcommand{\abs}[1]{\left\lvert#1\right\rvert}

\DeclareMathOperator*{\emb}{emb}

\newcommand{\Nat}{\mathbb{N}}

\newcommand{\wasstwo}[2]{\wasstwoop \left( #1, #2\right)}
\newcommand{\wasstwoop}{\operatorname{W}_2}

\newcommand{\relu}{\operatorname{ReLU}}

\newcommand{\paren}[1]{\left ( #1\right)}

\newcommand{\pushf}[2]{{#1}{}_{\#}{#2}}

\makeatletter
\def\mathcolor#1#{\@mathcolor{#1}}
\def\@mathcolor#1#2#3{%
  \protect\leavevmode
  \begingroup
    \color#1{#2}#3%
  \endgroup
}
\makeatother

\renewcommand{\eqref}[1]{{Eq.~\ref{#1}}}

\DeclareMathOperator*{\range}{Range}

\DeclareMathOperator{\lochom}{lochom}
\DeclareMathOperator{\diff}{dif}
\DeclareMathOperator{\locdiff}{locdif}

\DeclareMathOperator{\lip}{lip}

\DeclareMathOperator{\locbilip}{locbilip}

\newcommand\restr[2]{{
  \left.\kern-\nulldelimiterspace 
  #1 
  \vphantom{
  \big|
  } 
  \right|_{#2} 
  }}
  
\makeatletter
\@ifclassloaded{beamer}{}{
    \newtheorem{lemma}{Lemma}
    \newtheorem{theorem}{Theorem}
    \newtheorem{corollary}{Corollary}

    \theoremstyle{definition}
    \newtheorem{definition}{Definition}
    \newtheorem{example}{Example}
    
    \theoremstyle{remark}
    
}
\makeatother

\def\bszero{\boldsymbol{0}}

\makeatletter
\renewcommand{\boxed}[1]{\text{\fboxsep=.2em\fbox{\m@th$\displaystyle#1$}}}
\makeatother

\newcommand{\twopartpiecewise}[4]{\begin{cases} #1 & \text{if } #2 \\  #3 & \text{if } #4 \end{cases}}

\newcommand{\infseq}[3]{\paren{#1{}_{#2}}^\infty_{#2 = #3}}

\definecolor{orange}{RGB}{255,127,0}
\definecolor{blackchocolate}{HTML}{191102}
\definecolor{rosewood}{HTML}{510d0a}
\definecolor{rufous}{HTML}{A31B14}
\definecolor{citron}{HTML}{a29f15}
\definecolor{orangeyellow}{HTML}{f3b61f}
\definecolor{teagreen}{HTML}{bbd8b3}
\definecolor{gray}{RGB}{128,128,128}
\definecolor{lightcitron}{RGB}{189,187,91}
\definecolor{lightrufous}{RGB}{190,95,90}
\definecolor{steelblue}{HTML}{4C86A8}

\usepackage{tikz}
\usepackage{fonts}
\usepackage{gensymb}

\DeclareMathOperator{\myindex}{Ind}
\title{Deep Invertible Approximation of Topologically \\ Rich Maps between Manifolds}
\author{Michael Puthawala \thanks{Department of Mathematics and Statistics, South Dakota State University, Chicoine Hall, Box 2225, Brookings, SD 57007 \texttt{michael.puthawala@sdstate.edu}}
\and
Matti Lassas \thanks{Department of Mathematics and Statistics, University of Helsinki, FI-00014 Helsinki, Finland \texttt{matti.lassas@helsinki.fi}}
\and
Ivan Dokmani\'c \thanks{Department of Mathematics and Computer Science, University of Basel, Peterspl. 1, 4001 Basel, Switzerland \texttt{ivan.dokmanic@unibas.ch}}
\and
Pekka Pankka \thanks{Department of Mathematics and Statistics, University of Helsinki, FI-00014 Helsinki, Finland \texttt{pekka.pankka@helsinki.fi}}
\and
Maarten V. de Hoop \thanks{Computational and Applied Mathematics and Earth Science, Rice University, Houston, TX 77005, USA \texttt{mdehoop@rice.edu}}
}
\date{\today}

\begin{document}

\maketitle

\begin{abstract}
    How can we design neural networks that allow for stable universal approximation of maps between topologically interesting manifolds? The answer is with a coordinate projection. Neural networks based on topological data analysis (TDA) use tools such as persistent homology to learn topological signatures of data and stabilize training but may not be universal approximators or have stable inverses. Other architectures universally approximate data distributions on submanifolds but only when the latter are given by a single chart, making them unable to learn maps that change topology. By exploiting the topological parallels between locally bilipschitz maps, covering spaces, and local homeomorphisms, and by using universal approximation arguments from machine learning, we find that a novel network of the form $\mathcal{T} \circ p \circ \mathcal{E}$, where $\mathcal{E}$ is an injective network, $p$ a fixed coordinate projection, and $\mathcal{T}$ a bijective network, is a universal approximator of local diffeomorphisms between compact smooth submanifolds embedded in $\mathbb{R}^n$. We emphasize the case when the target map changes topology. Further, we find that by constraining the projection $p$, multivalued inversions of our networks can be computed without sacrificing universality. As an application, we show that learning a group invariant function with unknown group action naturally reduces to the question of learning local diffeomorphisms for finite groups. Our theory permits us to recover orbits of the group action. We also outline possible extensions of our architecture to address molecular imaging of molecules with symmetries. Finally, our analysis informs the choice of topologically expressive starting spaces in generative problems. 
\end{abstract}

\section{Introduction}

Topology is key in machine learning applications, from generative modeling, classification and autoencoding to applications in physics such as gauge field theory and occurrence of topological excitations. Here we describe a neural network architecture which is a universal approximator of locally stable maps between topological manifolds. In contrast to classical universal architectures, like the multilayer perceptron (MLP), the architecture studied in this work is built with forward and inverse stability in mind. We emphasize that our network applies to the case when the topology of the manifolds are not known a priori.

Proving that specific deep network architectures are universal approximators of broad classes of functions have been long studied with much progress in recent years. Beginning with \cite{cybenko1989approximation} and \cite{hornik1991approximation}, shallow networks formed with $\relu$ or sigmoid activation functions were shown to be universal approximators of continuous functions on compact subsets in $\R^n$. Recently an effort has emerged to extend the existing work to more specialized problems where other properties, for example monotonicity \cite{huang2018neural} or stability of inversion \cite{puthawala2022universal}, are desired in addition to universality.

In parallel with these developments, manifold learning has emerged as a vibrant subfield of machine learning. Manifold learning is guided by the manifold hypothesis, the mantra that ``high dimensional data are usually clustered around a low-dimensional manifold'' \cite{tenenbaum1998mapping}. This in turn begat the subfield of manifold learning \cite{belkin2004semi,belkin2006manifold,cayton2005algorithms,fefferman2016testing,kumar2017improved,lei2020geometric,rifai2011manifold,roweis2000nonlinear,saul2003think,tenenbaum2000global,vincent2003manifold,weinberger2006unsupervised}. The guiding principle is that a useful network needn't (and often shouldn't!) operate on all possible values in data space. Instead, it is better to use the ansatz that one should manipulate data that lies on or near a low-dimensional manifold in data space. The manifold hypothesis has helped guide network design in numerous applications, for example in classification (see e.g. \cite{lecouat2018semi,naitzat2020topology,rifai2011manifold,simard1991tangent,yu2019tangent}) where data belonging to a fixed label is conceived of as being on a common manifold, as well as both generative and encoding tasks, (see e.g. \cite{brehmer2020flows,bronstein2017geometric,cunningham2020normalizing,ganea2018hyperbolic,krioukov2010hyperbolic,nickel2017poincare,nickel2018learning,sarkar2011low,shao2018riemannian} and  \cite{child2020very,dai2019diagnosing,dinh2016density,kingma2019introduction}) where the manifold hypothesis is used as an ``existence proof'' of a low-dimensional parameterization of the data of interest. In the context of inverse problems, the manifold hypothesis can be interpreted as a statement that forward operators map low-dimensional space to the high-dimensional space of all possible measurements \cite{alberti2020inverse,angles2018generative,anirudh2018unsupervised,behrmann2018invertible,jin2017deep,kothari2021trumpets,kruse2021benchmarking,narnhofer2019inverse,siahkoohi2020faster,whang2020approximate}. The hypothesis has also been used in Variational Autoencoder (VAE) and Generative Adversarial Networks (GAN) architectures for solving inverse problems \cite{angles2018generative,anirudh2018unsupervised,ardizzone2018analyzing,behrmann2018invertible,hand2018phase,jin2017deep,kruse2021benchmarking,shah2018solving,siahkoohi2020faster}.

A natural question at the intersection of universality efforts and manifold learning is the following. What kinds of architecture are universal approximators of maps between manifolds? In particular, which networks are able to learn functions on manifolds if the manifolds are not known ahead of time? Some existing methods that operate on the level of manifolds use tools such as persistent homology to learn homology groups of data \cite{bruel2019topology,hofer2020topologically,hofer2020graph,hofer2019connectivity,hofer2017deep,moor2020topological}. In this work we look to learn mappings that are locally inverse stable. This condition is not generally true for homotopies between manifolds, and so we are unable to use many tools in TDA tools. Other approaches not utilizing TDA exist and are able to learn mappings between submanifolds, but only when the manifolds to be learned have simple topology. That is, they only apply to manifolds that are given by a single chart and so can't learn mappings that change topology \cite{brehmer2020flows,puthawala2022universal,kothari2021trumpets}. Because of this latter limitation they are only able to apply to problems where the starting and target manifold are different embeddings of the same manifold. In the context of generative models, this means that one has to get the topology of the starting space ``exactly right'' in order to learn a pushforward mapping that has inverse stability.

In order to combine the best of these two approaches, in this paper we look to learn mappings that are universal approximators of mappings between manifolds that are locally diffeomorphisms but globally complex. A practical gain of such an approach can be seen in the generative context. One no longer has to get the topology of the starting space `exactly right' to insure that a suitable stable forward map exists. We also present a result that establishes mathematical parallels between problems in invariant network design and cryogenic electron microscopy (cryo-EM) where learning a mapping that changes topology arises naturally.

\subsection{Network Description}

For families of functions $\cH$ and $\cG$ with compatible domain and range we use a well-tuned shorthand notation and write $\cH \circ \cG \coloneqq \set{h \circ g\colon h \in \cH, g \in \cG}$ to denote their pair-wise composition. We introduce \emph{extension-projection networks} which are of the form 
\begin{align}
    \label{eqn:network-definition}
    T\circ p\circ E,\quad \hbox{where }T\in\cT, \ E\in \cE
\end{align}
and $\cT \subset C(\Rea^{n_{\ell}},\Rea^{n_{\ell}})$ is a family of homeomorphism
$T:\Rea^{n_{\ell}}\to \Rea^{n_{\ell}}$, $p$ is a fixed projector, and 
$\cE$ is a family of networks of the form $\cE \coloneqq \cT^{n_{L}}_{L} \circ \cR_{L}^{n_{L-1},n_{L}} \circ \dots \circ \cT_1^{n_{1}} \circ \cR^{n_0,n_1}_1 \circ \cT^{n_{0}}_0$ where $\cR^{n_{\ell-1},n_\ell}_{\ell} \subset C(\Rea^{n_{\ell-1}}, \Rea^{n_{\ell}})$ are injective,  $\cT^{n_\ell}_{\ell} \subset C(\Rea^{n_{\ell}},\Rea^{n_{\ell}})$ are homeomorphisms, $L \in \Nat$, $n_0 = n$, $n_{L} = m$, and $n_{\ell} \geq n_{\ell-1}$ for $\ell = 1,\dots,L$. 

Examples of specific choices of $\cR$ in the definition of $\cE$ include zero padding, multiplication by full-rank matrix, injective $\relu$ layers or injective $\relu$ networks. Choices of the networks $\cT$ include Coupling flows \cite{dinh2014nice,dinh2016density,kingma2018glow} or Autoregressive flows \cite{kingma2016improving,huang2018neural}. For an extended discussion of these choices, please see \cite[Section 2]{puthawala2022universal}.

\subsection{Comparison to Prior Work}
\label{sec:prior-work}

In this subsection, we describe how this work is related to prior work in Topological Data Analysis (TDA), simplicial flow networks and group invariant/equivariant networks. 

\subsubsection{Topological Data Analysis}
\label{sec:prior-work:tda}

For a general survey of topological machine learning methods we refer to \cite{hensel2021survey}. Many approaches that stem from TDA use information gained directly from data, for example a persistence diagram \cite{carlsson2009topology}, and use this data as either a regularizer (in the form of a loss function) or as a prior for architecture design  \cite{bruel2019topology,hofer2020topologically,hofer2020graph,hofer2019connectivity,hofer2017deep,moor2020topological}. These works are closest to ours in terms of design, but fundamentally look to answer different questions than the ones studied in this work. TDA tools that use homology groups are sensitive enough to detect if two manifolds have the same homotopy type, but not sensitive enough to determine if they are homeomorphic.Two manifolds that are homeomorphic always have the same homotopy type, but the converse is not true. The converse direction follows by comparing the unit interval to a single point. 

Replacing the need for homeomorphisms with the need for homotopies means that TDA based approaches are unable to enjoy the theoretical guarantees of our work, in particular inverse stability and universality. TDA based works generally do not prove universality of their networks nor do they guarantee that network inversion is stable. 

\subsubsection{Simplicial Flow Networks}
\label{sec:prior-work:simplicial-flow-networks}

There is much work designing and developing the theory for networks that learn maps between graphs and simplicial complexes \cite{ebli2020simplicial,paluzo2021optimizing}. If two manifolds are triangulable, then functions between them can be learning a function between their triangulation with simplicial networks. The work presented here does not require access to a triangulation to be a universal approximator. Although all manifolds that we consider are triangulable, a fact which is important for our proof, the triangulation is not actually necessary for the statement of the final theorem \ref{thm:approximation-of-locdiff}. Simplicial networks have the advantage that they do not require any dimensionality restrictions, while our results do.

\subsubsection{Group Invariant/Equivariant networks}

Finally, we describe a connection between group invariant networks and the work presented here. 
Partially motivated by the success of convolutional networks applied to image data, there has recently been much development and analysis of group invariant and group equivariant networks \cite{birrell2022structure,yarotsky2022universal,murphy2018janossy,bietti2021sample,puny2021frame,maron2018invariant,liu2021algorithms}. For a group $\Sigma$ with action $g_\sigma\colon \cM_1 \to \cM_1$ defined over $\cM_1$ a function $f \colon \cM_1 \to \cM_2$ is called $\Sigma$-invariant if $f(g_\sigma x) = f(x)$ for all $x \in \cM_1$, $\sigma \in \Sigma$.

If the symmetry group $\Sigma$ is known then we can design a network that enforces and exploits this symmetry by architecture choice \cite{bietti2021sample} or by averaging over the group action \cite{puny2021frame}. Conceptually, we may consider both similar in so far as they both approximate an $f$ of the form $f = \restr{f}{\cM_1 / \Sigma} \circ \pi_\Sigma \colon \cM_1 \to \cM_2$  where $\pi _\Sigma \colon \cM_1 \to \cM_1 / \Sigma$ projects $\cM_1$ onto the orbits of $\cM_1$ under $g_\Sigma$.

Learning functions $f\colon \cM_1 \to \cM_2$ that are invariant w.r.t. some group action $g_\Sigma$ on $\cM_1$ is closely related to the idea of learning local homeomorphisms between $\cM_1$ and $\cM_2$. If $\Sigma$ is finite with group action $g_\Sigma$ that is $d$-to-one and smooth enough then, as presented in Lemma \ref{lem:symm-and-quotient-manifs}, $g_\Sigma$ is a local diffeomorphism. Thus it can be approximated using the networks studied here. Further, because our networks are built with inversion in mind we can recover orbits of $\cM_1$ under the action of $\Sigma$, see Corollary \ref{cor:recov-of-group-action}. In this way, our network is a universal solver for the `finite blind invariance problem.' To this point, we don't intend to offer our network as a drop-in replacement to existing $\Sigma$ invariant networks. Instead it offers a different perspective on a closely related problem.

\subsection{Our Contribution}

In this work we show that extension-projection networks of the form \ref{eqn:network-definition} are universal approximators of local diffeomorphisms between smooth compact manifolds. The problem has two parts. In the first, we show that by extending the analysis of \cite{puthawala2022universal}, we can universally approximate any embedding of a smooth, compact (topologically complex) $n$ manifold. In the second, we show that we approximate mappings that globally change topology between manifolds, but locally are diffeomorphisms. The latter part is the more difficult and novel, and so is the main focus of this work. By using topological arguments, in particular the parallels between local diffeomorphisms and covering maps, we find that local diffeomorphisms (locally one-to-one, globally many-to-one) can be lifted to diffeomorphisms (globally one-to-one) in a sufficiently high dimensional space. Further, this lifting can always be done so that the inverse lifting (a projection) is a coordinate projection. This is the main content of Theorem \ref{thm:covering-map-decomp}. The projection is simple enough that it may be inverted. By approximating the lifting, projection, and one final embedding with existing networks, we find that the architecture end-to-end is a universal approximator of local diffeomorphisms and maintains the novel inversion property, Theorem \ref{thm:approximation-of-locdiff}.

We also consider applications of extension-projection networks in the design of group-invariant networks, the choice of starting spaces in generative problems, and describe a promising connection between the problem of cryogenic electron microscopy (cryo-EM) when the sample to be imaged possesses an unknown symmetry.

\section{Theoretical development}

The manifolds that we consider here are smooth, compact $n$-dimensional and embedded in $\R^m$ for some $m > n$. In particular, we pose the question of approximating a surjection $f \colon \cM_1 \to \cM_2$ with a network of the form \ref{eqn:network-definition}. Before presenting our results, we point readers to Appendix \ref{sec:review-of-top-terms} for a definition of terms used in this manuscript. 

\subsection{Bistable approximations and local homeomorphisms}
\label{sec:bistable-approx-and-topology}

We wish to introduce architectures that are universal approximators and have properties which are necessary or desirable in practice. In particular, we consider approximations that locally have an inverse that is stable. We call this mode of approximation a \emph{bistable approximation}. {To define this concept, we recall 
that when $\cM\subset \R^m$ is a $C^\infty$-smooth submanifold, the reach of $\cM$, denoted $\hbox{Reach}(\cM)$, is the supremum
of all $r>0$ such that for all $x$ in the $r$-neighborhood $U_r\subset \R^m$  of $\cM$ there is
the unique nearest point $y\in \cM$. We denote the nearest point by $y=P_{\cM}(x)$. We note that for 
$0<r<\hbox{Reach}(\cM)$, the map $P_{\cM}:U_r\to \cM$  is $C^\infty$-smooth \cite{MR4156994}. For a point $x\in U_r$
the pair $(y,v)$ of the nearest point $y=P_{\cM}(x)$ and the normal vector $v=x-P_{\cM}(x)\in N_y\cM$ of $\cM$ at $y$ form tubular coordinates of the point $x$.}

\begin{definition}[Bistable approximation]
\label{def:bistable-approximation}
    Let $\cF \subset \locdiff^{1}(\cM_1,\R^{m_2})$, and $g \in C(\cM_1,\cM_2)$ where $g$ is surjective for $n$-dimensional submanifolds $\cM_1 \subset \R^{m_1}$ and $\cM_2 \subset \R^{m_2}$. We say that $\cF$ has a bistable uniform approximator of $g$, if there is an $M > 0$, so that for any 
     {$0<\epsilon<\hbox{Reach}(\cM_2)$}  there is an $f \in \cF$ so that for all $x \in \cM_1$ the following hold
    \begin{align}
        \label{eqn:def:bistable-approximation}
        \norm{f(x) - g(x)}_{\R^{m_2}} < \epsilon, \quad
        \norm{\nabla f(x)}_{T_x\cM_1\times{  \R^{m_2}}} \leq M,\quad 
        \norm{\paren{\nabla {(P_{\cM_2}\circ f)}(x)}^{-1}}_{T_{P_{\cM_2}(f(x))}\cM_2\times T_x\cM_1} \leq M.
    \end{align}
   
    For a family of functions $\cG$, we say that $\cF$ is a bistable uniform approximator of $\cG$, if it has a bistable uniform approximator for each $g \in \cG$. We call a sequence $\infseq fn1 \subset \cF$ a bistable uniform approximating sequence if there is an $M > 0$ and $\infseq \epsilon n1$ such that $\lim_{n \to \infty} \epsilon_n = 0$ and for all $x\in \cM_1$ and $n\in \mathbb Z_+$ we have $\norm{\paren{\nabla  (P_{\cM_2}\circ f_n)}^{-1}}_{ {  T_{P_{\cM_2}(f_n(x))}\cM_2}\times T_x\cM_1} \leq M$, $\norm{\nabla f_n(x)}_{T_x\cM_1\times\R^{m_2}} \leq M$ and $\norm{f_n(x) - g(x)}_{\R^{m_2}} < \epsilon_n$.
\end{definition}


A neural network architecture $\cF$ undergoing training to approximate a function $g$ can be formalized as a bistable uniform approximating sequence $\infseq fn1$, where $f_n \to g$ in the limit. In this formalization, we let 
$f_1$ be the network after being trained for one epoch, $f_2$ be the network trained for two epochs, etc. The three terms in Eqn. \ref{eqn:def:bistable-approximation} have a natural interpretation in this formalism.

Requiring $\norm{f(x) - g(x)}_{\R^{m_2}} < \epsilon$ forces an approximating sequence to be a uniform approximator on compact sets and justifies saying that $f$ approximates $g$.

The $\norm{\nabla f(x)}_{T_x\cM_1\times\R^{m_2}} \leq M$ term requires that good approximations to $g$ are stable, and in particular penalizes convergence to discontinuous or `kinked' functions. If, for a given value of the weights, a network is Lipschitz (as is the case for deep feed-forward networks with $\relu$ or sigmoid activation functions), then uniform convergence of $f_n$ to a $g$ with unbounded gradient implies that $\norm{\nabla f_n(x)}$ diverges as $n \to \infty$. If the weights of a neural network diverge as the network is trained, this is generally thought to be undesirable. Thus requiring $\norm{\nabla f_n(x)}$ to be bounded permits only `good' convergence.

The $ \norm{\paren{\nabla {(P_{\cM_2}\circ f)}(x)}^{-1}}_{{  T_{P_{\cM_2}(f(x))}\cM_2}\times T_x\cM_1}$ term enforces a locally stable inverse of an approximation and has a meaning in the context of Bayesian Uncertainty Quantification (Bayesian UQ). In Bayesian UQ, the uncertainty associated with an approximation is evaluated by computing a change of variables term. To estimate this term, it is necessary to compute, or suitably approximate, the inverse gradient. Thus, by requiring $ \norm{\paren{\nabla   (P_{\cM_2}\circ f)(x)}^{-1}}_{T_{P_{\cM_2}(f(x))\cM_2}\times T_x\cM_1}$ to be bounded, we require that this change of variables has determinant bounded away from zero, and so the model admits Bayesian UQ that doesn't deteriorate in the limit.
    
For a family $\cF \subset \locdiff^{1}(\cM_1,\R^{m_2})$, the notation $\overline \cF$ denotes the closure of $\cF$ under bistable uniform approximation. That is, $\overline \cF = \set{g \in C(\cM_1,\cM_2) \colon \cF \text{ has a bistable uniform approximator for } g}$.

With the definition given above we present our first result, which describes what kinds of functions admit bistable uniform approximators by any function class.

\begin{lemma}[Closure of Bistable Uniform Approximation]
    \label{lem:closure-of-bistable-approx}
    If $\cM_1$ is compact then limit points of bistable uniform approximating sequences are locally bilipschitz. 
\end{lemma}
    
The proof for Lemma \ref{lem:closure-of-bistable-approx} is given in the Appendix in \ref{sec:proof:lem:closure-of-bistable-approx} and implies the following

\begin{corollary}[Bistable Uniform Approximations are Local Homeomorphisms]
    \label{cor:bistable-unif-approxs-are-loc-hom}
    If $\cM_1$ is compact, then $\overline{\locdiff^{1}(\cM_1,{ \cM_2})} \subset \lochom(\cM_1,\cM_2)\cap W^{1,2}(\cM_1,\cM_2)$.
        The set $W^{1,2}(\cM_1,\cM_2)$ refers to functions that are weakly differentiable with $L^2(\cM_1,\cM_2)$ derivative.
\end{corollary}

The proof for Corollary \ref{cor:bistable-unif-approxs-are-loc-hom} is given in the Appendix in \ref{sec:proof:cor:bistable-unif-approxs-are-loc-hom}. The fact that $\overline{\locdiff^{1}(\cM_1,\R^{m_2})} \subset \lochom(\cM_1,\cM_2)$ is key to the subsequent developments of this paper. 
Before moving onto those developments, we give an example of manifolds to illustrate the difference between $\hom(\cM_1,\cM_2)$ and $\lochom(\cM_1,\cM_2)$, a distinction which is key to understanding the contribution of our paper.

\begin{figure}
    \centering
    \includegraphics[width=1.00\linewidth]{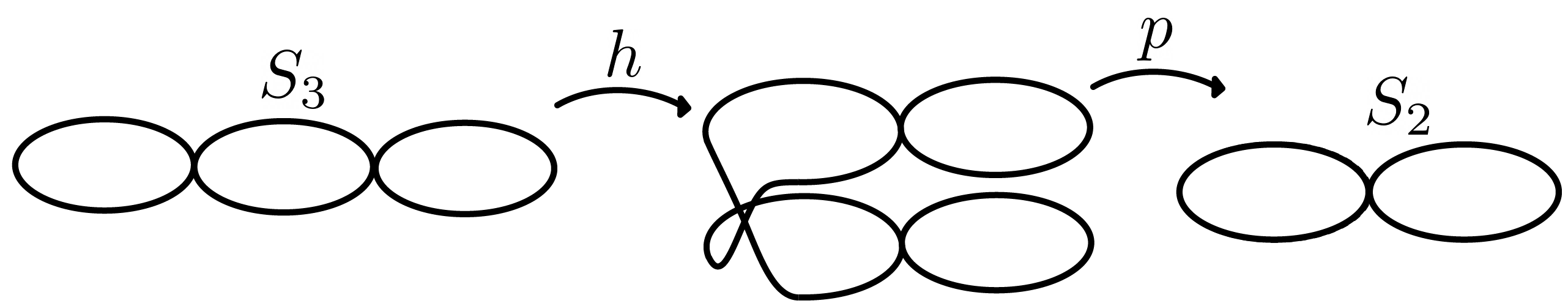}
    \caption{In the space $\R^5$ there is a genus three surface $\Sigma_3$
    and in the space $\R^3$ there is a genus two surface $\Sigma_2$ such that
    projection $p_3:\R^5\to \R^3$ is a covering map (and thus a local diffeomorphism) $p_3:\Sigma_3\to \Sigma_2$.
These surfaces can be obtained by taking curves that are deformed to surfaces by replacing all points of the curves by circles. The figure gives idea of the construction of such curves. In the figure we have a deformation $h:\R^3\to \R^3$ that maps a curve $S_3\subset \R^3$ to a self-intersecting curve $h(S_3)\subset \R^3$
and projection $p:\R^3\to \R^2$ that maps $h(S_3)$ to a curve $S_2\subset \R^2$. We add a forth dimension (time) in the space $\R^3$ that contains $S_3$, and modify $h$ in the time-variable to obtain
a diffeomorphism $\tilde h:\R^4\to \R^4$ so that $\tilde S_3=\tilde h(S_3\times \{0\})$
 is  a non-self-intersecting by curve $\R^4$ which projetion to the 3-dimensional space $\R^3$ is $h(S_3)$.
We add one more dimension in all Euclidean spaces and obtain a diffeomorphism $ \tilde H:\R^{4+1}\to \R^{4+1}$, $\tilde H(x,t)=(\tilde h(x),t)$ and the projection $p_3:\R^{4+1}\to \R^{2+1}$ such that $p_3\circ \tilde H:\tilde S_3\times \{0\}\to S_2\times \{0\}$.  Finally, we replace all points of the curves by circles (analogously to the Penrose diagrams in general relativity) to deform the 
curves $\tilde H(\tilde S_3)$ and $S_2\times \{0\}$ to surfaces $\Sigma_2\subset \R^{2+1}$ and  $\Sigma_3\subset \R^{4+1}$ having the genus 2 and 3, respectively.}
    \label{fig:local-hom-aid}
\end{figure}

\begin{example}[Examples of Homeomorphic and Locally Homeomorphic Manifolds]
    \label{examp:hom-vs-lochom}
    Homeomorphisms preserve homotopy type so a surface of genus 3, $S_3$, and genus 2, $S_2$, are not homeomorphic. Nevertheless, there is e.g. a two-sheet covering of $S_2$ by $S_3$, and so there is a local homeomorphism from $S_3$ to $S_2$. A visualization of this covering is given in Figure \ref{fig:local-hom-aid}. The covering is done in two steps, in the first the A local homeomorphism from $S_3$ to $S_2$ is given explicitly in Appendix \ref{sec:proof:lem:closure-of-bistable-approx}.
\end{example}

Lemma \ref{lem:closure-of-bistable-approx} and Corollary \ref{cor:bistable-unif-approxs-are-loc-hom} show that limiting sequences of bistable uniform approximating sequences are always at least local homeomorphisms, and are in general weakly differentiable. A natural question is if they are always local diffeomorphisms. They are not. An example of a nonsmooth function that can be approximated bistably is given in Appendix \ref{sec:bistable-approx-of-non-smooth}.

Any $C^1$ diffeomorphism can be uniformly approximated by $C^\infty$ diffeomorphisms \cite[Theorem 2.7]{hirsch2012differential}, hence uniform approximation results true of $C^1$ diffeomorphisms, are automatically true of $C^\infty$ ones as well. Further, from \cite{muller2014uniform}, we have that a homeomorphism can be uniformly approximated by a diffeomorphism if and only if the homeomorphism is isotopic to a diffeomorphism, as in the example given in Appendix \ref{sec:bistable-approx-of-non-smooth}.

A limit homeomorphism $f_\infty$ of a sequence $(f_i)$ is weakly differentiable in the Sobolev sense, see Corollary \ref{cor:approx-to-maps-betwen-non-local-homeo-manifs}. The sign of the Jacobian determinant $\det Df_\infty$ carries topological information on the orientation of the mapping. Under weaker notions of converge this analytical information on the Jacobian determinant is lost. For example, by a result of Hencl and Vejnar \cite{hencl2016sobolev}, there exists a homeomorphism $f\colon \mathbb{R}^4 \to \mathbb{R}^4$, which belongs to the Sobolev space $W^{1,1}$ but for which sets $\{\det Df>0\}$ and $\{\det Df<0\}$ have positive measure. This homeomorphism cannot be approximated by diffeomorphisms in the Sobolev norm of $W^{1,1}$. Indeed, if such an approximating sequence $(f_i)$ would exist, then it would have a subsequence $(f_{i_j})$ having the property that $\det(Df_{i_j}) \to \det(Df)$ almost everywhere, but this is not possible as functions $\det Df_{i_j}$ do not change sign. 

The homeomorphism of Hencl and Vejnar is an intricate construction based on a reflection in a set containing a positive measure Cantor set and we do not discuss the construction of Hencl and Vejnar in more detail here. We merely note that homeomorphisms $f\times \rm{id} \colon \R^4 \times \R^k \to \R^4 \times \R^k$ for $k>0$ yield similar examples in all dimensions above four. In lower dimensions $n\le 3$ construction of such homeomorphisms is not possible by a result of Hencl and Mal\'y \cite{hencl2010jacobians} which states that $W^{1,1}$-homeomorphisms in these dimensions have the property that their Jacobians do not change sign.

\subsection{Measures on manifolds with nontrivial topology}

The following result says that if we could get access to \textit{some} manifold $\cM_1$ which is homeomorphic to a target manifold $\cM_2$, than we can learn to approximate measures on $\cM_2$ using a network, under some dimension requirements.

\begin{theorem} 
    \label{thm:manifold-universality}
    Let $n_0 = n, n_L = m$, $\cM \subset \R^n$ be a compact $k$ manifold for $n \geq 3k+1$, $\mu \in \cP(\cM_1)$, and be an absolutely continuous\footnote{Here the absolute continuity means that $\pushf{\phi}{\mu}$ is absolutely continuous w.r.t. Lebesgue measure in $\R^k$ w.r.t. each chart $\phi$ of $\cM_1$.}. Further let, for each $\ell = 1,\dots,L$, $\cE_\ell^{n_{\ell-1}, n_{\ell}} \coloneqq \cT^{n_\ell}_\ell \circ \cR^{n_{\ell-1}, n_\ell}_\ell$ for $\ell = 1,\dots,L$, where $n_\ell\ge 3n_{\ell - 1}+1$, 
    $\cR^{n_{\ell - 1}, n_\ell}_\ell\subset C^1(\R^{n_{\ell-1}},\R^{n_{\ell}})$ is a family of injective maps that contains a linear map and
     the families $\cT^{n_\ell}_\ell$ are dense in $\hbox{Diff}^2(\R^{n_\ell})$ (e.g.,  $\cT^{n_\ell}_\ell$ is a family of bijective networks \cite[Section 2]{puthawala2022universal}). Finally let $\cT^{n}_{0}\subset \hbox{Diff}^2(\R^{n_0})$ be distributionally universal family and $f\in \emb^1(\cM_1, \R^m)$. Then, there is a sequence of {$\set{E_{i}}_{i = 1,2,\dots} \subset \cE_L^{n_{L-1}, m} \circ \dots \circ \cE_1^{n_1, n} \circ \cT^{n}_0$} such that
    \begin{align}
        \label{eqn:qual-univ:qual}
        \lim_{i \to \infty}
        \wasstwo{\pushf{f}{\mu}}{\pushf{E_i}{\mu}} = 0.
    \end{align}
\end{theorem}

The proof of Theorem \ref{thm:manifold-universality} is given in the Appendix in \ref{sec:proof:thm:manifold-universality}. This theorem has the following corollary that says that, morally, the generative problem of learning a $\nu$ with support $\cM_2$ can be solved if we had access to a $\cM_1$ diffeomorphic to $\cM_2$.

\begin{corollary}
    \label{cor:topology-is-needed}
    Let $\nu$ be a Borel measure in $\Rea^{m_2}$ which support is a subset of some smooth compact $n$ submanifold $\cM_2$ of $\R^{m_2}$ and let $\nu$ be absolutely continuous w.r.t. Riemannian measure of $\cM_2$. 
    If a submanifold $\cM_1$ in $\R^{m_2}$ is smooth and diffeomorphic to $\cM_2$ for $m_1\gg n$ and $m_2 \gg n$, then the Trumpet architecture \cite{kothari2021trumpets} pushes forward the uniform distribution on $\cM_1$ arbitrarily close to $\nu$ in Wasserstein distance.
\end{corollary}

The 0proof of Corollary \ref{cor:topology-is-needed} is given in the Appendix in \ref{sec:cor:topology-is-needed}. This corollary says that one can use existing architectures, e.g. \cite{kothari2021trumpets}, to solve generation problems, provided that we know the topology of $\cM_2$, so that we can construct a suitable $\cM_1$. In this sense, although the embedding of $\cM_2$ doesn't need to be known exactly, its topology must be completely understood. Furthermore, provided we can obtain universality with respect to the relevant function classes, than we can get universality in the sense of pushforward of measure.

\begin{corollary}[Bistable Uniform Approximation Implies Pushforward Universality]
    \label{cor:bistable-unif-approx-implies-pushf-univ}
    Let $m_2 \geq 3n+1$, $\cM_1\subset \R^{m_1}, \cM_2\subset \R^{m_2}$, and let $\cF \subset \locdiff^{1}(\cM_1,\R^{m_2})$ be a bistable uniform approximator for $\cG \subset C(\cM_1,\cM_2)$. Then for any $\mu \in \cP(\cM_1)$ absolutely continuous and $g \in \cG$, there is a sequence $\infseq fi1\subset \cF$ so that $
        \lim_{i \to \infty} \wasstwo{\pushf{g}{\mu}}{\pushf{f_i}{\mu}} = 0$.
\end{corollary}

The proof of Corollary \ref{cor:bistable-unif-approx-implies-pushf-univ} is given in the Appendix in \ref{sec:proof:cor:bistable-unif-approx-implies-pushf-univ}. The above results show that if $\cT \circ p \circ \cE$ is a bistable uniform approximator to $\cF$, then questions of learning to approximate measures supported on $\range(f)$ for $f \in \cF$ are solved. Thus, for the remainder of the section, we focus on the problem of showing that $\cT \circ p \circ \cE$ is a bistable uniform approximator with respect to the largest class $\cF$ possible.

\subsection{Covering maps and learning topology}

Here we present a topological theorem. This result has a technical appearance, but is the major work horse for the subsequent developments in the paper.

\begin{theorem}[Covering Map Decomposition]
    \label{thm:covering-map-decomp}
    Let $\cM_1\subset \Rea^{m_1}$ and $\cM_2\subset \Rea^{m_2}$ be smooth compact $n$-manifolds, where $\cM_2$ is triangulable, and let $p\colon \Rea^{m_2} \times \R^{2n+1} \to \Rea^{m_2}$ be the projection to the first coordinate. Then each local diffeomorphism $g \colon \cM_1\to \cM_2$ there exists an embedding $h \colon \cM_1 \to \cM_2 \times \R^{2n+1}$ for which $g = p \circ h$ and so that the restriction $p|_{h(\cM_1)} \colon h(\cM_1)\to \cM_2$ is a covering map. Moreover, if $\cM_2$ is given a triangulation $K$ that is fine enough, then we may fix $h$ satisfying the following condition: for each vertex $v$ of the triangulation $K$, the preimage $p^{-1}(v)$ is $\{v\}\times \{0,1\ldots, d\}\times \{0\}^{2n}$, where $d$ is the degree of the map $g$.
\end{theorem}

The proof of Theorem \ref{thm:covering-map-decomp} is given in the Appendix in \ref{sec:proof-of-decomposition-theorem}. A difficulty in proving Theorem \ref{thm:covering-map-decomp} is that $p$ is not merely a map between abstract spaces, but rather a projector in the ambient space. The fact that such a projector exists is non-obvious when $\cM_1$ or $\cM_2$ are, for example, knotted\footnote{For an example of a 3D printed model of an internally knotted torus artistically rendered by Carlo S\`euin, see \url{http://gallery.bridgesmathart.org/exhibitions/2011-bridges-conference/sequin}}. 

Theorem \ref{thm:covering-map-decomp} allows us to approximate a local diffeomorphism $g$ by approximating $h$, a diffeomorphism. We can then compose this approximation with a projection $p$ which only depends on the degree of $g$. This leads to the following theorem.

\begin{theorem}[Bistable Approximations of Local Diffeomorphisms]
    \label{thm:approximation-of-locdiff}
    Let $\cM_1 \subset \R^{m_1}$ and $\cM_2 \subset \R^{m_2}$ be compact submanifolds, $m_2 \geq 3n+1$, $\cE \subset \emb^1\paren{\cM_1,\R^{m_2 + 2n+1}}$ be a uniformly bistable approximator of each smooth embedding $h \colon \cM_1 \to \R^{m_2 + 2n+1}$, and $\cT \subset \diff^1\paren{\R^{m_2},\R^{m_2}}$ be a uniformly bistable approximator of $\diff^1(\R^{m_2},\R^{m_2})$. Then, for any $g:\cM_1\to \cM_2$, and finite $Y \subset \cM_2$, there is a bistable uniform approximating sequence $\infseq fi1 \subset \cT \circ p \circ \cE$ for $g$. Further, if the degree of $g$ is $d$ then for $f_i = T_i\circ p \circ E_i$, the set
    \begin{align}
        \label{eqn:thm:approximation-of-locdiff:x-i-def}
        X_i \coloneqq \set{\argmin_{x \in \cM_1}\norm{E_i(x) - (T_i^{-1}(y),j,\set{0}^{2n})}_2 \colon y \in Y, j \in \set{1,\dots,d}}
    \end{align}
    satisfies
    \begin{align}
        \label{eqn:thm:approximation-of-locdiff:unif-conv-of-inverses}
        d_H(X_i,g^{-1}(Y)) \to 0
    \end{align}
    as $i \to \infty$ where $d_H$ is the Hausdorff distance.
\end{theorem}
The proof of Theorem \ref{thm:approximation-of-locdiff} is given in the Appendix in \ref{sec:proof-of-approximation-of-locdiff}. The set $X_i$ defined in Eqn. \ref{eqn:thm:approximation-of-locdiff:x-i-def} can be computed in closed-form provided that $E_i$ and $T_i$ admit closed-form inverses. This, combined with Eqn. \ref{eqn:thm:approximation-of-locdiff:unif-conv-of-inverses}, means that we can compute arbitrarily good approximations to the multi-values inverses of $g$ on any point in $Y$. Theorem \ref{thm:approximation-of-locdiff} says that local diffeomorphisms can always be learned in a bistable way using expansive-projectors. What about manifolds that don't admit local diffeomorphisms? We explore this question in the following

\begin{corollary}[Approximations to Maps Between Non Locally Homeomorphic Manifolds]
    \label{cor:approx-to-maps-betwen-non-local-homeo-manifs}
    Let $\cM_1 \subset \R^{m_1}$ and $\cM_2 \subset \R^{m_1}$ be smooth compact manifolds, and $\lochom(\cM_1,\cM_2)$ be empty. Let $g \colon \cM_1 \to \cM_2$ be continuous and surjective. There are no bistable uniform approximating sequences of $g$.
\end{corollary}

The proof of Corollary \ref{cor:approx-to-maps-betwen-non-local-homeo-manifs} is given in the Appendix in \ref{sec:proof:cor:approx-to-maps-betwen-non-local-homeo-manifs}. Corollary \ref{cor:approx-to-maps-betwen-non-local-homeo-manifs} shows that the \emph{only} functions between smooth compact manifolds that can be approximated bistably (by anything) but are not approximated by the network lie in $\lochom(\cM_1, \cM_2) \setminus \locdiff(\cM_1, \cM_2)$.
\section{Applications and Implications}
\label{sec:applications-and-implications}

In this section, we describe how the networks studied in this work are naturally connected to various other problems in machine learning, as well as an application in cryo-EM.

\subsection{Group Invariant Networks}
\label{sec:applications-and-implications:group-invariant-networks}

Let $\Sigma$ be a group with action $g_\sigma \colon \cM_1 \to \cM_1$ for each $\sigma \in \Sigma$. When constructing a group invariant network $E \colon \cM_1 \to \cM_2$, the task is to approximate an $f \colon \cM_1 \to \cM_2$ by building a network (via choice of architecture or else averaged training) so that $E(g_\sigma x) = E(x)$ for all $\sigma \in \Sigma$ and $x \in \cM_1$. If $f$ is $\Sigma$ invariant as well then enforcing $\Sigma$ invariance of $E$ does not harm approximation and improves both training and generalization error in theory \cite{bietti2021sample} and in practice \cite{chaman2021truly}.

For each $x \in \cM_1$ we call the orbit of $x$ the set $\text{orbit}(x) \coloneqq\set{g_\sigma x \colon \sigma \in \Sigma}$. We let $\cM_1 / \Sigma$ denote the quotient of $\cM_1$ by the group action, that is, the orbit space of $\cM_1$ by $\Sigma$. A group on $\cM$ is called free if for any $x \in \cM_1$, $g_\sigma x = x$ implies that $\sigma$ is the identity. It is called smooth if the map $g_\sigma(x)$ is smooth as a function of $x$ for each $\sigma \in \Sigma$. If a finite $\Sigma$ is smooth and free, then $\cM_1 / \Sigma$ is a smooth manifold \cite[Theorem 21.13]{lee2013smooth}. This suggests that there is a connection between learning $f \colon \cM_1 \to \cM_2$ that is $\Sigma$ invariant, and learning a `symmetrized' modification of $f$ defined between $\colon \cM_1 / \Sigma \to \cM_2$. When the group action induces constant-sized orbits (a stronger condition than $\Sigma$ acting freely) then this is indeed the case.

\begin{lemma}[Symmetrization as a Quotient Manifolds]
    \label{lem:symm-and-quotient-manifs}
    Let $\Sigma$ be a continuous finite group on compact $n$ manifold $\cM_1$ so that the orbit of $x$ is the same size for all $x \in \cM_1$, and let $f \colon \cM_1 \to \cM_2$ be a continuous $\Sigma$ invariant smooth map. Then the quotient space $\cM_1 / \Sigma$ is an $n$ manifold,
    \begin{enumerate}
        \item let $\pi_\Sigma\colon \cM_1 \to \cM_1 / \Sigma$ take each point to its orbit and $\restr f{\cM_1 / \Sigma}\colon \cM_1 / \Sigma \to \cM_2$ so that $\restr f{\cM_1 / \Sigma}(\text{orbit}(x)) = f(x)$, then
        \vspace{-3mm}
        \begin{align}
            \label{eqn:symmetrization-identity}
            f = \restr f{\cM_1 / \Sigma}\circ \pi_\Sigma,
        \end{align} \vspace{-3mm}
        
        \item if both $f$ and $\Sigma$ are smooth, $f$ is surjective and $d$-to-one where $\abs{\Sigma} = d$, then $f \in \locdiff(\cM_1,\cM_2)$.
    \end{enumerate}
\end{lemma}

The proof of Lemma \ref{lem:symm-and-quotient-manifs} is given in the Appendix in \ref{sec:proof:lem:symm-and-quotient-manifs}. The r.h.s. of Eqn. \ref{eqn:symmetrization-identity} can be used to construct a $\Sigma$ invariant network in general. See, e.g. the symmetrization operator studied in \cite[Sec. 3.4]{birrell2022structure} or group convolution \cite[Sec. 4.3]{bronstein2021geometric}, both of which play a similar role to the $\pi_\Sigma$ here.

The final point of Lemma \ref{lem:symm-and-quotient-manifs} gives conditions under which a $\Sigma$ invariant function $g$ is a local diffeomorphism and so can be approximated bistably by the networks considered here. Combining this with Theorem \ref{thm:approximation-of-locdiff} yields a result that says that we can back out the group action of $\Sigma$ from the $g$ approximation, without knowledge of $\Sigma$. The proof of Corollary \ref{cor:recov-of-group-action} is given in the Appendix in \ref{sec:proof:cor:recov-of-group-action}.

\begin{corollary}[Recovery of Group Action]
    \label{cor:recov-of-group-action}
    Let $h\colon \cM_1 \to \cM_2$ be a smooth, surjective, $d$-to-one, $\Sigma$ invariant function for $n$ submanifolds $\cM_1 \subset \R^{m_1}$ and $\cM_2 \subset \R^{m_2}$, and let $Y\subset \cM_2$ be finite. Then there is a sequence $\infseq fi1 \subset \cT\circ p \circ \cE$ that is a bistable uniform approximator for $h$, and for each $y \in Y$, $f_i^{-1}(y)$ converges to a $\Sigma$ orbit in Hausdorff distance.
\end{corollary}

\subsection{Choice of Starting Space}
\label{sec:applications-and-implications:choice-of-starting-space}

In a generation problem, the goal is to approximate a probability distribution $\nu$ over some subset $\cX$ of $\R^m$ given samples $X$ from $\nu$. This can be solved by fixing a base distribution $q$ over some simpler subset $\cZ$ of $\R^n$ and using a neural network to learn $f \colon \R^n \to \R^m$ so that $\pushf f q \approx \nu$ \cite{brehmer2020flows,bronstein2017geometric,cunningham2020normalizing,ganea2018hyperbolic,krioukov2010hyperbolic,nickel2017poincare,nickel2018learning,sarkar2011low,shao2018riemannian}. This leads to the question: how should we choose $\cZ$ to allow for maximal flexibility of $\cX$?

The classification theorem \cite[Theorem 77.5]{munkres2000topology} says that each connected compact $2$ manifolds is homeomorphic to $S^2$, the $n$-fold torus, or the $m$-fold projective plane. Further, orientable surfaces can be classified by genus alone while a nonorientable surface is covered by its orientation covering \cite[Theorem 15.41]{lee2013smooth} itself an orientable surface. Moreover, 
by \cite[Example 1.41 and Page 157]{hatcher2005algebraic}, there exists some covering map
from an oriented Riemannian surface $S_m$ of genus $m$ to an oriented Riemannian surface $S_n$ of genus $n$
if and only if $m=(n-1)k+1$, as the following lemma shows.

\begin{lemma}[Local Diffeomorphisms Between Surfaces]
    \label{lem:local-dif-surfaces}
    Let  $m=(n-1)k+1$, where $n,m,k\in \mathbb Z_+$, $k\ge 2$ and $d\ge 5$. Moreover,
    let  $S_n$  and  $S_m$  be oriented Riemannian surface of genus $n$ and $m$, respectively.
    Then  $S_n$  and  $S_m$  can be embedded in $\R^d$ so that the restriction of the projection
    $p:\R^3\times \R^{d-3}\to \R^3,$ $p(x,y)=x$ defines a $k$-to-1 covering map $p|_{S_m}:S_m\to S_n$.
    In particular, $p|_{S_m}:S_m\to S_n$ is a surjective local diffeomorphism.
\end{lemma}
The proof of Lemma \ref{lem:local-dif-surfaces} is given in the Appendix in \ref{sec:proof:lem:local-dif-surfaces}. This shows that such a covering map exists, however the existence of such a covering map does not guarantee that
if $S_n$ and $S_m$ are submanifold of an Euclidean space that the covering map extends
to continuous map of the Euclidean space. In the next lemma we prove when $S_n$ and $S_m$
are embedded in the Euclidean space appropriately, the covering map is realized by projection.

This lemma can be combined with the following theorem which gives a general construction for local diffeomorphisms using projectors between manifolds $\cM_1$ and $\cM_2$ by passing through diffeomorphisms.

\begin{theorem}[Covering Maps as Projections]
    \label{thm:covering-maps-as-projs}
    Let $\cM_1 \subset \R^{m_1}$, $\cM_2 \subset \R^{m_2}$, $\tilde \cM_1\subset \R^{\tilde m_1}$ and $\tilde \cM_2\subset \R^{\tilde m_2}$ be compact smooth submanifolds where $\tilde m_1>m_1$ and $\tilde m_2>m_2$. Suppose further that both $\cM_1$ and $\tilde \cM_1$ as well as $\cM_2$ and $\tilde \cM_2$ are diffeomorphic and there is a  projection $p\colon \R^{\tilde m_1}\to \R^{\tilde m_2}$ so that the restriction of $p|_{\tilde \cM_1}$ is a covering map
    $p|_{\tilde \cM_1}:\tilde \cM_1\to \tilde \cM_2$.
    Then there is a linear injective map $J_1\colon\R^{m_1}\to \R^{\tilde m_1}$,
    a diffeomorphism $T_1\colon\R^{\tilde m_1}\to \R^{\tilde m_1}$,
    an integer $k\ge \max(\tilde m_2,3m_2+1)$,
    a linear injective map $J_2\colon\R^{\tilde m_2}\to \R^{k}$,
    a diffeomorphism $T_2\colon\R^{k}\to \R^{k}$, and 
    a projection $p_2\colon \R^{k}\to \R^{m_2}$
    such that $
        f=p_2\circ T_2\circ J_2\circ p\circ T_1\circ J_1 \in \locdiff(\cM_1,\cM_2)$
    is a $k$-to-1 covering map $f:\cM_1\to \cM_2$.
\end{theorem}

The proof of Theorem \ref{thm:covering-maps-as-projs} is given in the Appendix in \ref{sec:proof-of-covering-maps-as-projs}. Theorem \ref{thm:covering-maps-as-projs} gives a recipe for extending this to the case when $S_3$ is a surface with two knotted handles, that is a torus that is embedded in expansive-projectors non-standard \cite{osada2016handlebody}.
As products of covering maps are a covering map,
Theorem \ref{thm:covering-maps-as-projs} and
Lemma \ref{lem:local-dif-surfaces} imply the following. Let $m=k!+1$, $N\subset \R^{d_1}$ be a compact connected
submanifold and $S_m\subset \R^{d_2}$ be an oriented Riemannian surface of genus $m$.
Next, we consider  the model space $N\times S_m\subset \R^{m_1}$, $m_1=d_1+d_2$ and its
possible images in the composition networks of the form \eqref{eqn:network-definition}.
Let  $\cM_2 \subset \R^{m_2}$ be a submanifold that that is diffeomorphic 
to $N\times S_n$, where $S_n$ is an oriented Riemannian surface of genus $n$ with $n\le k+1$.
Then by Theorem \ref{thm:covering-maps-as-projs} there is a composition 
$f:= p_2\circ T_2\circ J_2\circ p\circ T_3\circ J_3:\R^{m_1}\to\R^{m_2}$ of projections, global diffeomorphisms and linear injective maps such that
$f:N\times S_m\to N\times S_n$ is a surjective local diffeomophism.
This means that a combination of two composition networks of the form \eqref{eqn:network-definition} can
can map the model manifold $N\times S_m$ to any manifold diffeomorphic
to $N\times S_n$ with $n\le k+1$. In particular, if we consider the first betting number $b_1(N)$
of the manifold $N$ (i.e. the number of handles), the composition map $f$ can map $N\times S_n$
to manifolds having any 1st Betti number  $n\cdot b_1(N)$  with $n\le k+1$.

\subsection{Cryogenic Electron Microscopy}
\label{sec:applications-and-implications:cryo-em}

Cryogenic electron microscopy (cryo-EM) is a molecular imaging technique where samples (molecules) are suspended in vitreous ice and  tomographically imaged with an electron microscope (2017 Nobel Prize in Chemistry).
Like in traditional tomography one images the three-dimensional sample with two-dimensional projections (slices) but with unknown orientations. Orientations are then modeled as random samples from the rotation group $\mathrm{SO}(3)$, often from the Haar (uniform) measure \cite{bendory2020single,frank2008electron}. The task is to recover the underlying molecular density up to natural global symmetries. Mathematically, this can be modeled as the problem of recovering the orbit of samples under the rotation group, $\mathrm{SO}(3)$ \cite{liu2021algorithms}. 

In addition to the emergence of $\mathrm{SO}(3)$ as a natural group in this problem, the molecule  often has its own discrete symmetries which means that the orbit to recover is topologically different from $\mathrm{SO}(3)$.\footnote{For example bacteriophages have a natural axis of discrete rotational symmetry. COVID viruses such as  COVID-19 are spherical with non-symmetric spikes.} The presence of a symmetry in the sample, encoded by a group $\Sigma$, means that the natural problem is to recover orbits of $\mathrm{SO}(3) / \Sigma$, where $\Sigma$ is unknown. This task is a natural setting for the analysis in this manuscript, see Lemma \ref{lem:symm-and-quotient-manifs} and Theorem \ref{cor:recov-of-group-action}. 

We remark that the results described in Section \ref{sec:applications-and-implications:group-invariant-networks} apply only when $\Sigma$'s group action produces orbits of constant size. This is not the case for some natural settings. Indeed in the setting of imaging a bacteriophage, rotation out the axis of symmetry fixes the points at the poles, and so is not free. Thus, quotients of manifolds by $\Sigma$ do not yield manifolds, but so-called orbifolds \cite[Chapter 13]{thurston2014three}. Extending the analysis of the manuscript to apply to the case when the target space is an orbifold will be the subject of future work.

\section{Conclusion}

In this work we showed that extension-projection networks are universal approximators of local diffeomorphisms between smooth compact manifolds. In particular, we showed that we can approximate mappings that globally change topology between manifolds. We found that local diffeomorphisms can always be lifted to diffeomorphisms in a sufficiently high dimensional space. By approximating this lifting and a subsequent projection, we found that extension-projection networks are end-to-end universal approximators of local diffeomorphisms while maintaining a novel inversion property. Finally, we considered applications where our extension-projection networks can be used.

\section{Acknowledgements}

M.P. was supported by the CAPITAL Services. M.L. was  supported by Academy of Finland, grants 284715, 312110. P.P. was supported by Academy of Finland grant 332671. I.D. was supported by the European Research Council Starting Grant 852821---SWING. M.V. dH. gratefully acknowledges support from the Department of Energy under grant DE-SC0020345, the Simons Foundation under the MATH + X program, and the corporate members of the Geo-Mathematical Imaging Group at Rice University. 

\bibliographystyle{plain}
\bibliography{references.bib}

\appendix 

\section{Definition of Terms}
\label{sec:review-of-top-terms}

A Hausdorff space is a topological space so that distinct points are always separated by disjoint neighborhoods. Metric spaces are always Hausdorff, and so is $\R^n$ with the usual metric. 

A $n$ manifold, notated $\cM$, is a Hausdorff space with countable basis such that each point $x$ in $\cM$ has a neighborhood that is homeomorphic with an open subset of $\R^n$. 

A $n$ submanifold is a manifold that is embedded in $\R^m$ for some $m$.

Given two topological spaces $X$ and $Y$, we notate $\hom(X,Y)$ as the space of homeomorphisms between $X$ and $Y$. That is, the set of functions $X \to Y$ which are continuous and bijective with continuous inverse. 

Given two topological spaces $X$ and $Y$, we notate $\lochom(X,Y)$ as the space of local homeomorphisms between $X$ and $Y$. That is, the set of functions $f \colon X \to Y$ so that $f$ maps open subsets of $X$ to open subsets of $Y$ and for every $x \in X$, there is an open neighborhood $U \subset X$ so that $\restr{f}{U}\colon U \to f(U)$ is a homeomorphism.
Observe that if $X$ is compact and non-empty and $Y$ is connected, then all local homeomorphism
$f:X\to Y$ are surjective maps.

Given a Riemannian $n$ manifold $\cM$ we use the notation $d_\cM(\cdot,\cdot)$ to refer to the geodesic metric.

Given two $n$ manifolds with geodesic metrics $\cM_1$ and $\cM_2$, we notate $\locbilip(\cM_1,\cM_2)$ as the space of locally bilipschitz functions $f\colon \cM_1 \to \cM_2$ which are local homeomorphisms and for which there is an $L > 0$ so that for all $x\in \cM_1$ there is an open $U_x \subset \cM_1$ so that for all $x' \in U_x$ $\frac 1L d_X(x,x') \leq d_Y(f(x),f(x')) \leq L d_X(x,x')$.

Given two smooth manifolds $X$ and $Y$, a smooth map $f \colon X \to Y$
is a diffeomorphims if it is bijection and its inverse is a smooth map.
We denote $\diff(X,Y)$ as the space of diffeomorphisms between $X$ and $Y$. 
Moreover, we notate $\locdiff(X,Y)$ as the space of local diffeomorphisms between $X$ and $Y$. 
That is, the set of functions $f \colon X \to Y$ so that $f$ maps open subsets of $X$ to open subsets of $Y$ and for every $x \in X$, there is an open neighborhood $U \subset X$ so that $\restr{f}{U}\colon U \to f(U)$ is a diffeomorphism. All local diffeomorphism are
local homeomorphism and thus, if $X$ is compact and non-empty and $Y$ is connected, then all local diffeomorphism
$f:X\to Y$ are surjective maps.

We call a function $f$ an embedding and denote it by  $f\in \emb(X,Y)$ if $f : X \to Y$ is continuous, injective, and $f^{-1}\colon f(X) \to X$ is continuous\footnote{Note that if $X$ is a compact set, then continuity of the of $f^{-1}\colon f(X) \to X$ is automatic, and need not be assumed \cite[Cor.\  13.27]{MR2548039}. Moreover, if $f:\R^n\to \R^m$ is a continuous injective map that satisfies $|f(x)|\to \infty$ as $|x|\to \infty$, then by
\cite[Cor.\ 2.1.23]{mukherjee2015differential} the map $f^{-1}\colon f(\R^n) \to \R^n$ is continuous.}. Also 
we denote by $ \emb^k(\R^n,\R^m)$ the set of maps $f\in \emb(\R^n,\R^m)\cap  C^k(\R^n,\R^m)$ which differential $df|_x:\R^n\to \R^m$ is injective at all points $x\in \R^n$.

Given a smooth $n$ manifold and point $x \in \cM$, $T_x\cM$ denotes the tangent space of $\cM$ at $x$. $T\cM$ denotes the tangent space of $\cM$.

Two homeomorphisms $\phi$ and $\psi$ are said to be isotopic if there exists a continuous $\Phi\colon [0,1]\times \cM \to \cM$ so that $\Phi(t,\cdot)$ is a homeomorphism for each $t$, and so that $\Phi(0,\cdot) = \phi$, and $\Phi(1,\cdot) = \psi$.

Let $X$ be a topological space. We call $X$ a \emph{polyhedron} if there exists a simplicial complex $K$ and a homeomorphism $h \colon \abs{K} \to X$. The pair $(K,h)$ is called a \emph{triangulation} of $X$. For a well-written, high level description of triangulation of manifolds, see \cite{manolescu2016lectures} or \cite[Chapter 7]{rotman2013introduction}. A key result is that all smooth topological manifolds are triangulable. Further, every topological $n$ manifold is triangulable if $n < 4$, but counter examples exist for $n \geq 4$.

In Definition \ref{def:bistable-approximation}, $\paren{\nabla f(x)}^{-1}$ refers to the inverse of a gradient matrix, not the gradient of the inverse mapping. Similarly the norm $\norm{\cdot}_{T_x\cM_1\times{\R^{m_2}}}$ denotes, e.g., the operator (matrix) norm.

\section{Proofs}
\label{sec:proofs-of-main-results}

\subsection{Helper Lemma}
Before we proceed with the proof of our main results, we first present the following helper lemma, which allows us to compare the euclidian distance between points with the geodesic distance. It lets us compare euclidian distance (as vectors in $\R^m$) and arclength geodesic distances (as points on an $n$-manifold $\cM$) between two points. This presentation and Lemma are not original, and are taken from \cite[Section 3]{bernstein2000graph}.

Let $\cM$ be a smooth $n$ manifold embedded in $\R^m$. We define the minimum radius of curvature $r_0(\cM)$ as 
\begin{align*}
    \frac 1{r_0(\cM)} = \max_{\gamma,t}\norm{\ddot\gamma(t)}_{\R^n}
\end{align*}
where $\gamma\colon D \to \R^m$ varies over all unit-sphere geodesics in $\cM$, and $t$ varies over $D$. The minimum branch separation $s_0(\cM)$ is the largest positive number for which 
\begin{align*}
    \norm{x-y}_{\R^m} < s_0(\cM) \text{ implies that } d_\cM(x,y) \leq \pi r_0(\cM).
\end{align*}

\begin{lemma}[Comparing Euclidian and Geodesic Distances]
    \label{lem:comp-euclidian-and-geodesic}
    Let $\cM$ be a smooth compact $n$ manifold that connects points $x$ and $y$ with geodesic $\gamma$ of length $\ell$. Then $r_0(\cM) > 0$ and $s_0(\cM) > 0$ both exist. Further, if $\ell \leq \pi r_0(\cM)$, then 
    \begin{align*}
        2r_0(\cM)\sin(\ell / 2r_0(\cM)) \leq \norm{x - y}_{\R^m} \leq \ell.
    \end{align*}
\end{lemma}
The Lemma above is the same as in \cite[Lemma 3]{bernstein2000graph}. We refer the reader to that work for a proof. If we denote the geodesic distance $d_\cM \colon \cM \times \cM \to \R$ as
\begin{align*}
    d_M(x,y) \coloneqq \inf_\gamma\set{\text{length}\paren{\gamma}}
\end{align*}
where $\gamma$ varies over the smooth paths connecting $x$ and $y$, then Lemma \ref{lem:comp-euclidian-and-geodesic} implies that when $d_\cM(x,y) \leq \pi r_0(\cM)$,
\begin{align}
    \label{eqn:geodesic-euclidian-equivalence}
    (2/\pi)d_\cM(x,y) \leq \norm{x-y}_{\R^m} \leq d_\cM(x,y).
\end{align}
When  $\cM$ is compact, this implies that there is a constant $c_\cM>0$ such that 
for all $x,y\in \cM$,
\begin{align}
    \label{eqn:geodesic-euclidian-equivalence global}
   c_\cM d_\cM(x,y) \leq \norm{x-y}_{\R^m} \leq d_\cM(x,y).
\end{align}

\subsection{Proof Of Lemma \ref{lem:closure-of-bistable-approx}}
\label{sec:proof:lem:closure-of-bistable-approx}

In this subsection we present the proof of Lemma \ref{lem:closure-of-bistable-approx}.

\begin{proof}
    Let $\infseq fn1$ be a bistable uniform approximating sequence converging to $g\colon \cM_1 \to \R^{m_2}$. We first show that $g$ is Lipschitz, and inverse Lipschitz when restricted to a small metric ball.
    
    Let $x_1,x_2 \in \cM_1$ be given, and close enough together so that Eqn. \ref{eqn:geodesic-euclidian-equivalence} applies. Then for any $n$, we have that
    \begin{align*}
        d_{\cM_2}(g(x_1),g(x_2)) &\leq \pi/2\norm{g(x_1) - f_n(x_1)}_{\R^m} + \pi/2\norm{g(x_2) - f_n(x_2)}_{\R^m} + d_{Y_n}(f_n(x_1),f_n(x_x))\\
        &\leq \pi M/2d_X(x_1,x_2) + \pi \epsilon_n
    \end{align*}
    where $Y_n$ denotes the manifold $f_n(\cM_1)$. If we let $n \to\infty$ then $n\to0$, and so $\norm{g(x_1) - g(x_2)} \leq M \norm{x_1 - x_2}$. Thus $g$ is Lipshitz with constant $\pi M/2$.
    
    Now we prove that $g$ is locally inverse Lipschitz. Let $n\in \mathbb Z_+$ and $x_0 \in \cM_1$. As $\norm{\paren{\nabla ( { P_{\cM_2}\circ f_n)}(x_0)}^{-1}}_{T_ {{ P_{\cM_2}\circ f_n}(x_0)}\cM_2 \times T_x\cM_1} \leq M$,  the inverse function theorem (see e.g. \cite[Theorem 17.7.2]{tao2009analysis}) 
  implies that  there is some neighborhood $V \subset \cM_1$ of $x_0$ so that $\restr{{P_{\cM_2}\circ f_n}}{V} \colon V \to W={P_{\cM_2}(f_n(V))}$ is a bijection and has a $C^1$-smooth 
  inverse $h_{n,V}:(\restr{P_{\cM_2}\circ f_n}{V})^{-1}:W\to V$ such that the norm of derivative of $h_{n,V}$
  is bounded by $M$. Consider now $x_1,x_2\in \cM_1$. Consider the geodesic $\gamma$ of $\cM_2$ that connects 
  $y_1=P_{\cM_2}\circ f_n(x_1)$ to  $y_2=P_{\cM_2}\circ f_n(x_2)$. By covering the geodesic $\gamma$ by neighborhoods
  $V_j$ where the inverse map $h_{n,V_j}$ are defined, we see that 
        \begin{align*}
       d_{\cM_1}(x_1, x_2)  \le  M
      d_{\cM_2}( P_{\cM_2}\circ f_n(x_1) ,P_{\cM_2}\circ f_n(x_2)) . 
           \end{align*}

   Next, let $x_1,x_2 \in \cM_1$ { and  assume the $\epsilon_n$ is so small that
    $\epsilon_n<\hbox{reach}(\cM_2)/2$}.
Then,      \begin{align*}
    \norm{
    { P_{\cM_2}\circ f_n(x_1) - P_{\cM_2}\circ f_n(x_2)}} &\geq c_{\cM_2}
  d_{\cM_2}( P_{\cM_2}\circ f_n(x_1) ,P_{\cM_2}\circ f_n(x_2)) \\
  &\geq c_{\cM_2} M^{-1}d_{\cM_1}(x_1, x_2)   \\ &
    \geq c_{\cM_2} M^{-1} \norm{x_1 - x_2}
      \end{align*} 
  Recall that by  the definition of $P_{\cM_2}$, the point
  $P_{\cM_2}\circ f_n(x)$ is the nearest point of $\cM_2$ to $f_n(x)$.
 Hence, as $|f_n(x)-g(x)|<\epsilon_n$ and $g(x)\in \cM_2$, we have 
 $$|f_n(x)-P_{\cM_2}\circ f_n(x)|\leq |f_n(x)-g(x)|<\epsilon_n.$$ Hence,
  \begin{align*}
 & \norm{{  f_n(x_1) - f_n(x_2)}} 
 \geq \norm{{ P_{\cM_2}\circ f_n(x_1) - P_{\cM_2}\circ f_n(x_2)}}-2\epsilon_n \geq c_{\cM_2} M^{-1} \norm{x_1 - x_2}-2\epsilon_n.
      \end{align*}
    
    Then we have
    \begin{align*}
        d_{\cM_2}(g(x_1),g(x_2)) &\geq \norm{g(x_1) - g(x_2)}_{\R^m}\\
        &\geq \norm{f_n(x_1) - f_n(x_2)}_{\R^m} - \norm{f_n(x_1) - g(x_1)}_{\R^m} - \norm{f_n(x_2) - g(x_2)}_{\R^m}\\
        &>c_{\cM_2} M^{-1} \norm{x_1 - x_2}-4\epsilon_n
    \end{align*}
    The above is true for every $n$, and $\epsilon_n \to 0$ as $n \to \infty$, thus $\norm{g(x_1) - g(x_2)} \geq c_{\cM_2}M^{-1} \norm{x_1 - x_2}$. Hence $g$ has an inverse map that is Lipschitz with constant $c_{\cM_2} M^{-1} $. This proves that $g \in \locbilip(\cM_1,\cM_2)$.
\end{proof}

\subsection{Proof of Corollary \ref{cor:bistable-unif-approxs-are-loc-hom}}
\label{sec:proof:cor:bistable-unif-approxs-are-loc-hom}

In this subsection we present the proof of Corollary \ref{cor:bistable-unif-approxs-are-loc-hom}.

\begin{proof}
    We prove the general fact that $\locbilip(X,Y) \subset \lochom(X,Y)$ when $X$ is  a compact  metric space and $Y$ is a metric space. Combining this with Lemma \ref{lem:closure-of-bistable-approx} yields   the first claim. Let $f \in \locbilip(X,Y)$
    and  $x_0\in X$, then there is some open set $U\subset X$ so that $\restr{f}{U}:U\to f(U)$ is bilipschitz.  Then there is a metric ball $B=B(x_0,r)$ such that $B\subset U$ and the set $K=\overline B$ is compact set. Observe that $K \subset U \subset X$. From this we have that $\restr{f}{K}$ is continuous (from forward Lipschitzness on $U$) and injective (from inverse Lipschitzness on $U$). Thus $\restr{f}{K}$ is a homeomorphism, and so for any open neighborhood $V \subset B$ of $x_0$,  the map $\restr{f}{V}$ is a homeomorphism as well.
    Therefore $f$ is a local homeomorphism, and so  $\locbilip(X,Y) \subset \lochom(X,Y)$
    
    From \cite[Pages 294 - 296]{evans1998partial} we have that if $U$ is open and bounded and $g$ lipschitz on $U$ then $g$ is weakly differentiable. $g$ is known to be locally (bi)Lipschitz on $X$ and but $X$ is compact hence $g$ is Lipshitz on $X$. Therefore $g$ is weakly differentiable, and so $\overline{\emb^{1}(\cM_1,\R^{m_2})} \subset W^{1,2}(\cM_1,\cM_2)$.
\end{proof}

\subsection{Bistable Approximation of Nonsmooth Functions}
\label{sec:bistable-approx-of-non-smooth}

\begin{example}[Bistable Approximation of Nonsmooth Functions]
    \label{examp:bistable-approx-of-non-smooth}
    \begin{align*}
        f(x) &= \twopartpiecewise{x}{x \leq 0}{ 2x}{x > 0},\quad 
        f_{\epsilon}(x) = \twopartpiecewise{\frac{1}{4\epsilon}(x + \epsilon)(x-\epsilon) + \frac32x + \frac12\epsilon}{x \in [-\epsilon,\epsilon]}{f(x)}{x \not \in [-\epsilon,\epsilon]}.
    \end{align*}
    A simple calculation shows that $\norm{\nabla f_\epsilon(x)} \leq 2$, $\norm{\nabla^{-1}f_\epsilon(x)} \leq 1$  and $f_\epsilon$ $f'_\epsilon(-\epsilon) = 1, f'_\epsilon(\epsilon) = 2$, and that $f_\epsilon \to f$ uniformly. 
\end{example}

\subsection{Proof of Theorem \ref{thm:manifold-universality}}
\label{sec:proof:thm:manifold-universality}

If $\cT^n_0$ is a universal approximator of $C([0,1]^n,[0,1]^n)$, then it must also be universal on $C(\cM_1,\cM_1)$ for compact $\cM_1$ as, after scaling, we can extend any $g \in C(\cM_1,\cM_1)$ to a $h \in C([0,1]^n,[0,1]^n)$ by the Tietze extension theorem \cite[Lemma 7.4]{madsen1997calculus}.

\begin{proof}
    The proof of this is a generalization of \cite[Theorem 3.10]{puthawala2022universal} where we don't have that $\mu \in \cP(\R^n)$, but rather that $\mu \in \cP(\cM_1)$.
    
    Let us consider a compact, smooth submanifold $\cM\subset \R^n$. 
    From \cite[Theorem 3.8]{puthawala2022universal}, we have that $\cI^k(\cM_1,\R^m) = \emb^k(\cM_1,\R^m)$, where $\emb^k(\cM_1,\R^m)$ is the set of all $C^k$-smooth embeddings $g:\cM_1\to \R^m$ and $\cI^k(\cM_1,\R^m)$ is the set of all $C^k$-smooth extendable embeddings $g_e:\cM_1\to \R^m$ that can be written as a composition $g_e=h\circ L|_{\cM_1}$ of a linear injective map $L:\R^n\to \R^m$ and the $C^k$-smooth diffeomorphism $h:\R^m\to \R^m$ of the entire space $\R^m$.
    Moreover, by \cite[Lemma 3.9]{puthawala2022universal}, we have that $\cE=\cE_L^{n_{L-1}, m} \circ \dots \circ \cE_1^{n_1, n}$ has the manifold embedding
property (MEP) w.r.t. the family $\cI^k(\R^n,\R^m)$ of extendable embeddings. Hence, just as in the proof of \cite[Theorem 3.10]{puthawala2022universal}, see that there exists of a map $\tilde E\in \cE$ and an a.c. measure $\mu'' \in \cP(\cM_1)$ such that $\wasstwo{\pushf{f}{\mu}}{\pushf{\tilde E}{\mu''}} < 2\epsilon_1$. By the distributional universality of the family $\cT^n_0$, that is a subset of the diffeomorphisms $T:\R^n\to \R^n$, there is a map $T \in \cT^n_0$ so that $\wasstwo{\mu''}{\pushf{T}{\mu}} < \epsilon_2$, and so $\wasstwo{\pushf{f}{\mu}}{\pushf{(\tilde E \circ T_0)}{\mu}} < 2\epsilon_1 + \epsilon_2\lip(\tilde E)$. Thus by choosing $\epsilon_1 < \frac{\epsilon}4$ and $\epsilon_2 < \frac{\epsilon}{2(1+\lip({\tilde E}))}$ yields a $E \in \cE^{n,m}$ so that $\wasstwo{\pushf{f}{\mu}}{\pushf{{E}}{\mu}} < \epsilon$, the result.
\end{proof}

\subsection{Proof of Corollary \ref{cor:topology-is-needed}}
\label{sec:cor:topology-is-needed}

In this section we present the proof of Corollary \ref{cor:topology-is-needed}.

\begin{proof}
    The uniform measure $\mu \in \cP(\cM_1)$ is absolutely continuous with respect to $\cM_1$. From $m_1 \gg n$, we have that the dimension condition of Theorem \ref{thm:manifold-universality} is met. Finally, the trumpet architecture of \cite{kothari2021trumpets} satisfies the necessary universality conditions, see \cite{puthawala2022universal}, hence we can apply Theorem \ref{thm:manifold-universality} and we obtain a sequence of $E_i$ so that $\lim_{i \to \infty} \wasstwo{\nu}{\pushf{E_i}{\mu}} = 0$.
\end{proof}

\subsection{Proof of Corollary \ref{cor:bistable-unif-approx-implies-pushf-univ}}
\label{sec:proof:cor:bistable-unif-approx-implies-pushf-univ}

In this section we present the proof of Corollary \ref{cor:bistable-unif-approx-implies-pushf-univ}. 

\begin{proof}
    Because $\cE$ is a bistable approximator, it is a uniform approximator with respect to $\cG$. Thus, by \cite{puthawala2022universal}, we have that $\cE$ has the MEP w.r.t. $\cG$. Thus, we can apply Theorem \ref{thm:manifold-universality}. This yields $\lim_{i \to \infty} \wasstwo{\pushf{f}{\mu}}{\pushf{E_i}{\mu}} = 0$.
    
\end{proof}

\subsection{Proof of Theorem \ref{thm:covering-map-decomp}}
\label{sec:proof-of-decomposition-theorem}

In this section we first present a lemma relating local homeomorphisms to covering maps. We use this lemma for our subsequent proof.
    
\begin{lemma}
\label{lem:hom-loc-to-covering-map}
    Let $\cM_1$ and $\cM_2$ be two compact Manifolds, and $f \colon \cM_1 \to \cM_2$ a continuous surjection. The following two are equivalent.
    \begin{enumerate}
        \item $f$ is a local homeomorphism,
        \item $\cM_1$ is a covering space with base space $\cM_2$, and $f$ a finite covering map.
    \end{enumerate}
\end{lemma}

\begin{proof}
    First we prove that (a) $\implies$ (b).
    
    Recall that a mapping is called proper if inverse images of compact sets are compact. Let $K \subset \cM_2$ be compact, we wish to show that $f^{-1}(K)$ is compact as well. Note that this part does not require $f$ to be local homeomorphisms, and is true more generally for continuous surjections between compact spaces. Compact sets are closed, and so $K$ is closed. Because $f$ is a surjection, $f^{-1}(K)$ is defined. Preimages of closed sets are closed, and so $f^{-1}(K)$ is closed. Finally, closed subsets of compact sets are compact, and so $f^{-1}(K)$ is compact. This shows that $f$ is proper.
    
    Lemma 2 of \cite{ho1975note} proves that a surjective proper local homeomorphism between $\cM_1$ and $\cM_2$ is a covering map. We now prove that the degree of $f$ is finite. Here we use the local homeomorphism property. Because $f$ is proper, $f^{-1}{(y)}$ is compact for any $y \in \cM_2$. By the local homeomorphism property of $f$, for each $x \in f^{-1}(y)$ there is an neighborhood of $x$, $U_x$ so that $\restr{f}{U_x}$ is a homeomorphism and, in particular, is injective. Thus $\set{U_x \colon x \in f^{-1}(y)}$ is an open covering of $f^{-1}(y)$ for which there is no subcover. This covering must be finite by compactness, and so too must be $f^{-1}{(y)}$.
    
    To prove that (b) $\implies$ (a), we prove that a covering map is a local homeomorphism. Let $x \in \cM_1$ be given, and $y \coloneqq f(x)$. Since $f$ is a covering map, there is a neighborhood $U_y$ of $y$ so that $f^{-1}(U_y)$ is a union of disjoint open sets. $x$ is in exactly one of these sets, denoted by $U_x$. Then $\restr{f}{U_x}$ is a homeomorphism. This works for any $x$, hence $f$ is a local homeomorphism.
\end{proof}

Now we present the proof of Theorem \ref{thm:covering-map-decomp}. We remark at the outset that this proof is presented as constructively as possible. Our hope is that the steps in the proof can inform network training.

\begin{figure}
    \centering
    \includegraphics[width=0.65\linewidth]{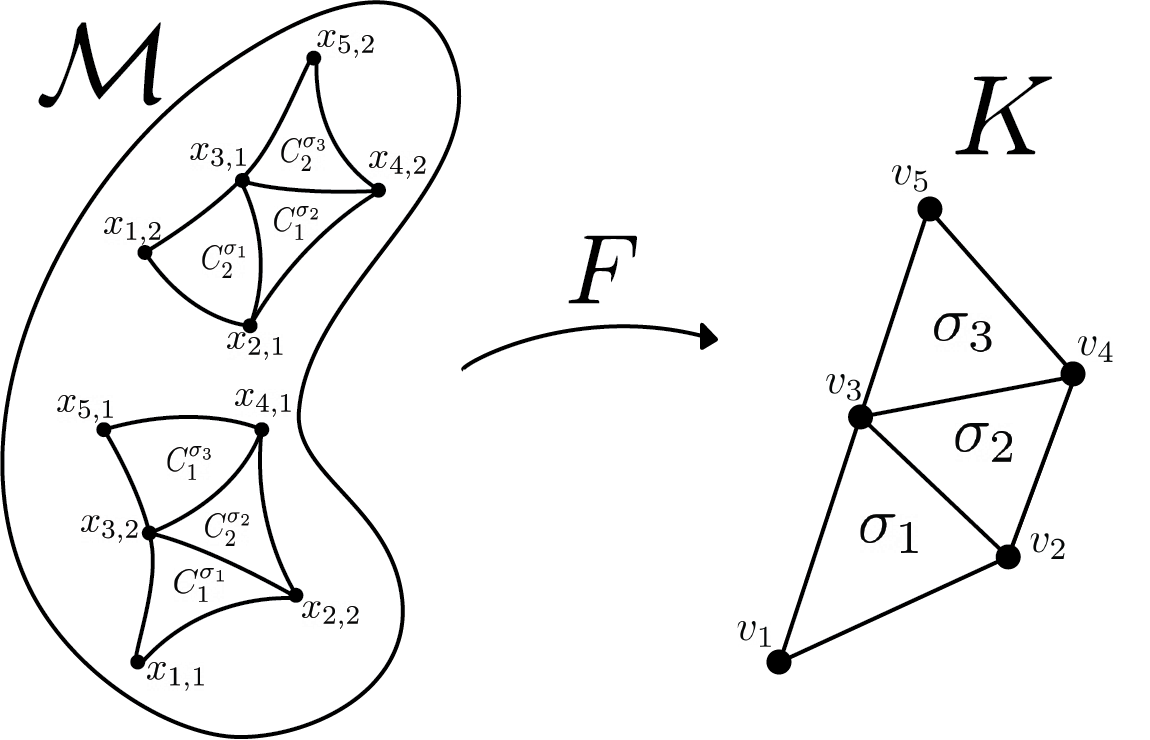}
    \caption{A sketch of the various quantities used in the proof of Theorem \ref{thm:covering-map-decomp}. This toy example provides a visual for the proof when $n = 2$ and $d = 2$. The blob on the left is $\cM$, and the simplices on the right are a subset of $K$. As indicated by the hashing, $F(C^{\sigma_1}_1) = F(C^{\sigma_1}_2) = \sigma_1$, and likewise $F(C^{\sigma_2}_1) = F(C^{\sigma_2}_2) = \sigma_2$ and $F(C^{\sigma_3}_1) = F(C^{\sigma_3}_2) = \sigma_3$. 
    Further, we have for each $i = 1,\dots, 5$ that $F(x_{i,1}) = F(x_{k,2}) = v_1$.}
    \label{fig:covering-sketch}
\end{figure}
    
\begin{proof}

The proof is quite involved and so we first provide a proof sketch highlighting the major steps of each section in the proof.

First, we show that for each $\sigma \in K$, there are $d$ disjoint compact patches $\set{C^\sigma_j}_{j=1}^d$ in $\cM$ such that $\restr{F}{C^\sigma_j}\colon C^\sigma_j \to \abs{\sigma}$ is a homeomorphism for each $j$. A sketch of the patches $C^\sigma_j$ are shown in Figure \ref{fig:covering-sketch}.

Second, we construct a map $g^\sigma\colon \cup_{j = 1}^d C^\sigma_j \to \abs{\sigma}\times [0,d]$  where $g^\sigma(x) = \paren{F(x),\cdot}$ that makes a $d$-tall `stack of pancakes' over $\sigma$. Each of the $d$ pancakes is $F(C^\sigma_j)$ for some $j$. We construct $g$ so that each pancake passes through $(v,i)$ for some $i \in 1,\dots,d$. We then construct a continuous map $g \colon \cM \to \abs{K}\times[0,d]$ so that $\restr{g}{\sigma} = g^\sigma$.

Third we show that $g$ is a bijection between $\cM_1$ and $\set{(v,i)}_{i = 1}^d$, and that $g$ can be continuously lifted to an injective $h \colon \cM \to C$ so that $h(x) = (g(x),\bszero)$. Thus $h$ is a continuous and injective map between compact sets, and so is an embedding of $\cM$ into $C$. Finally, we show that $p^{-1}(v)\cap h(\cM) = \set{(v,i,\bszero)}_{i = 1}^d$ for $\bszero \in \R^{2n}$.
--

Let the vertices of $V$ be enumerated as $V = \set{v_1,\dots,v_N}$ for some $N$, and for any $i \leq N$, let the set $X = \set{x_{i,j}}_{i = 1,j = 1}^{i = N, j = d} \subset \cM$ be the $d$ points in $\cM$ such that $F(x_{i,j}) = v_i$.

\begin{enumerate}
    \item[(1)] Let $\sigma \in K$ be given, and $\set{v_1,\dots,v_{n+1}} = \sigma\subset K$. 
    
    We will construct one $C^\sigma_j$ for each $x_{1,j}$. Let $j$ be given, then defined $r_j = \min_{x \in F^{-1}(v_1), x \neq x_{1,j}}\norm{x - x_{1,j}}$. Note that $r_j > 0$, and $x_{1,j} = B_{r_j}(x_{1,j}) \cap F^{-1}(v_1)$. Let 
    \begin{align*}
        C^\sigma_j = B_{r_j}(x_{1,j}) \cap F^{-1}(\abs{\sigma}).
    \end{align*}
    We will show that $\restr{F}{C^\sigma_j}$ is a homeomorphism. 
    
    It is clear that $\restr{F}{C^\sigma_j}$ is a local homeomorphism. Further, if $\sigma$ is small enough, then there is some $r' < r_j$ so that $\abs{\sigma} \subset F(B_{r'}(x_{1,j}))$. Therefore, we have that $\restr{F}{C^\sigma_j} \colon C^\sigma_j \to \abs{\sigma}$. This is guaranteed from the assumption that the diameter of the largest $\sigma \in S$ is sufficiently small. From Lemma \ref{lem:hom-loc-to-covering-map}, we then have that $\restr{F}{C^\sigma_j}$ is a covering map. Because $x_{1,j} = B_{r_j}(x_{1,j}) \cap F^{-1}(v_1)$, we can conclude that this covering map has degree one, and so $\restr{F}{C^\sigma_j}$ is not just a local homeomorphism, but a proper homeomorphism. As an immediate consequence of this, we have that $C^\sigma_j$ is compact and connected.
    
    \item[(2)]
    \begin{figure}
        \centering
        \includegraphics[width=0.65\linewidth]{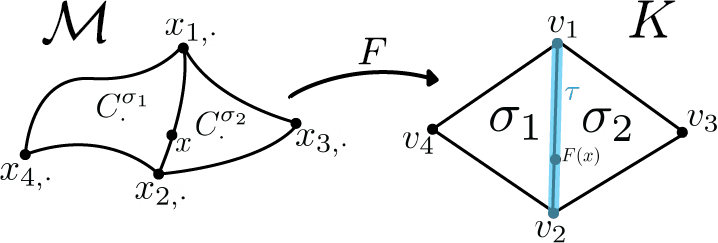} 
        \caption{Sketch of a toy example where a point $x$ is mapped to a simplex $\tau$ that is a common face of the simplices $\sigma_1$ and $\sigma_2$.}
        \label{fig:commonface-intersection}
    \end{figure}
    From the previous part we have that for a choice of $\sigma$ and $j$, that $C^\sigma_j$ is compact and $\restr{F}{C^\sigma_j}$ a homeomorphism. We also get that $C^\sigma_j$ are pairwise disjoint. 
    
    Let us define the set of closed patches $\cC = \set{C^\sigma_{j}\colon \sigma \in S^{\max{}}, j \in \set{1,\dots,d}}$ 
    where $S^{\max} \subset S$ are the maximal simplices of $S$. Now we introduce two index functions, $\myindex^S \colon \cC \to \set{1,\dots,d}$ and $\myindex^X \colon X \to \set{1,\dots,d}$. They are defined as:
    \begin{align}
        \myindex^S(C) &= j \text{ where } C = C^\sigma_j \in \cC\\
        \myindex^X(x) &= j \text{ where } x = x_{i,j} \in X.
    \end{align}
    We note that the choices of $\myindex^S$ and $\myindex^X$ are not unique. The $j$ index in the definition of $\cC$ and $X$ are both dummy indices. The key fact is that $C \mapsto (F(C),\myindex^S(C))$ and $x \mapsto (F(x),\myindex^X(x))$ are bijections.
    
    Let $y \in \abs{\sigma}$. We can express $y$ in barycentric coordinates in terms of the vertices of $\sigma$. Let $\paren{\lambda_v^\sigma(x)}_{v \in \sigma}$ be the unique scalars such that $F(x) = \sum_{v \in \sigma}\lambda_v^\sigma(x) v$, $\sum_{v \in \sigma}\lambda_v^\sigma(x) = 1$ and $\lambda_v^\sigma(x) \geq 0$. Now define the function $g^\sigma_k \colon C^\sigma_k\to\abs{\sigma} \times [0,d]$ as
    \begin{align}
        g^\sigma_k(x) = \paren{F(x), \sum_{v \in \sigma}\myindex^X\paren{F^{-1}(v)\cap C^\sigma_{k}}\lambda_v^
        \sigma(F(x))}
    \end{align}
    
    Let us briefly prove some properties about $g^\sigma_k$. First, recall from part (a) that $F^{-1}(v)\cap C^\sigma_k \in X$, hence the $\myindex^X\paren{F^{-1}(v)\cap C^\sigma_k}$ term is well-defined.
    Now we define
    \begin{align}
        \label{eqn:g-restr-to-sigma-def}
        g^\sigma \colon \bigcup^d_{k = 1}C^\sigma_k \to \abs{\sigma}\times [0,1], \quad 
        g^\sigma(x) = g^\sigma_{\myindex^S(C)}(x) \text{ where } x \in C \in \cC\\
        \label{eqn:g-def}
        g \colon \cM \to \abs{K} \times [0,1], \quad 
        g^\sigma(x) = \begin{cases}
            g^{\sigma_1}(x) &\text{if } F(x) \in \sigma_1\\
            \vdots &\\
            g^{\sigma_N}(x) &\text{if } F(x) \in \sigma_N
        \end{cases}
    \end{align}
    Note that from part (a), we have that $\set{C^\sigma_k}_{k = 1,\dots,d}$ are pairwise disjoint for any $\sigma$. Thus for any $x$, there is a unique $C \in \cC$ so that $x \in C$. This shows that Equation \ref{eqn:g-restr-to-sigma-def} is well-defined. Equation \ref{eqn:g-def}, on the other hand, may have problems. In particular, it is possible for $F(x) \in \sigma$ and $F(x) \in \sigma'$ for $\sigma \neq \sigma'$. We resolve this problem by showing that if $F(x) \in \sigma \cap \sigma'$, then $g^{\sigma}(x) = g^{\sigma'}(x)$. This resolves the problem.
    
    Let $x \in \cM$ be such that $F(x) \in \sigma \cap \sigma' = \tau$, see Figure \ref{fig:commonface-intersection} for an illustration. Because $F(x)$ lies in the convex hull of $\tau$, we can then conclude that 
    \begin{align}
        \label{eqn:lambda-simplex-simplification}
        \lambda^\sigma_v(x)  = \twopartpiecewise{\lambda^\tau_v(x)}{v \in \tau}{0}{v \not \in \tau}.
    \end{align}
    and likewise for $\sigma'$. Expanding the definition of $g^\sigma(x)$ we get that there is some $k$ so that
    \begin{align}
        g^\sigma(x) &= \paren{F(x), \sum_{v \in \sigma}\myindex^X\paren{F^{-1}(v)\cap C^\sigma_{k}}\lambda_v^\sigma(F(x))}\\
        &= \paren{F(x), \sum_{v \in \tau}\myindex^X\paren{F^{-1}(v)\cap C^\sigma_{k}}\lambda_v^\sigma(F(x)) + \sum_{v \in \sigma\setminus \tau}\myindex^X\paren{F^{-1}(v)\cap C^\sigma_{k}}\lambda_v^\sigma(F(x))}\\
        &= \paren{F(x), \sum_{v \in \tau}\myindex^X\paren{F^{-1}(v)\cap C^\sigma_{k}}\lambda_v^\tau(F(x))}.
    \end{align}
    Likewise, $g^{\sigma'}(x) = \paren{F(x), \sum_{v \in \tau}\myindex^X\paren{F^{-1}(v)\cap C^{\sigma'}_{k'}}\lambda_v^\tau(F(x))}$ for some $k'$. Finally, we have that $F^{-1}(\abs{\tau})\cap C^\sigma_{k} = F^{-1}(\abs{\tau})\cap C^{\sigma'}_{k'}$. Thus we have $F^{-1}(v)\cap C^\sigma_{k} = F^{-1}(v)\cap C^{\sigma'}_{k'}$. This proves that $g^\sigma(x) = g^{\sigma'}(x)$.
    
    Clearly, $g$ is continuous on the interior of every maximal simplex $\sigma \in S^{\max}$. The previous result shows that $g$ is continuous across the boundry of all simplices. Therefore, we also obtain continuity of $g$ by the gluing lemma.
    
    \item[(3)] 
    
    Before we construct $h$, we will show that $\restr{g}{X}$ is a bijection from $X$ onto $\set{(v,i)}_{i = 1,\dots,d}$. Let $x \in X$, then $F(x) = v' \in \sigma$ for some $\sigma \in S^{\max}$, and for that choice of $\sigma$, there is a unique $k$ so that $x \in C^\sigma_{k}$. Uniqueness follows from separation of $C^\sigma_k$. Hence 
    \begin{align*}
        g(x) = g^\sigma(x) = g^\sigma_k(x).
    \end{align*}
    Evaluating the value of $g^\sigma_k$ and using the identity \begin{align}
        \lambda^\sigma_{v'}(v) = \twopartpiecewise{1}{v' = v}{0}{v' \neq v}
    \end{align}
    we obtain
    \begin{align}
        g(x) = g^\sigma_k(x) &= \paren{v', \sum_{v \in \sigma}\myindex^X\paren{F^{-1}(v)\cap C^\sigma_{k}}\lambda_v^\sigma(v')}\\
        &= \paren{v',  \myindex^X(x) \lambda_{v'}(v') + \sum_{v \in \sigma, v\ \neq v'}\myindex^X\paren{F^{-1}(v)\cap C^\sigma_{k}}\underbrace{\lambda_v^\sigma(v')}_{ = 0 \text{ as }v \neq v'}}\\
        \label{eqn:g-on-vertices}
        &= \paren{v', \myindex^X(x)}.
    \end{align}
    If $x_{i,j}$ and $x_{i',j'}$ are such that $g(x_{i,j}) = g(x_{i',j'})$, then equality in the $v$ component of Eqn. \ref{eqn:g-on-vertices} proves that $i = i'$ and equality in the $\myindex^X$ component of Eqn. \ref{eqn:g-on-vertices} proves that $j = j'$. This proves that $g$ is injective on $X$. Both $X$ and $\set{(v,i)}_{v \in V, i = 1,\dots,d}$ have $\#(V) = d$ elements, therefore bijectivity follows from injectivity. 
    
    Now we turn to showing that $g$ can be continuously lifted to a function $h \colon \cM \to C$ that is injective on all of $\cM$, and so an embedding. Before starting with that proof, we remark on the non-injectivity of $g$. Because $g_k^\sigma$ is a homeomorphism, $\restr{g}{C^\sigma_k}$ is injective and so $g$ is locally injective, but $g$ may fail to be injective globally. Consider for example the case when $R^\sigma(v_1,1) = 1, R^\sigma(v_2,1) = 2, R^\sigma(v_1,2) = 2, R^\sigma(v_2,2) = 1$, $\tilde x_1 \in C^\sigma_1$ and $\tilde x_2 \in C^\sigma_2$ be such that $F(\tilde x_1) = F(\tilde x_2) = \frac12\paren{v_1 + v_2}$, then 
    \begin{align*}
        g(\tilde x_1) = \paren{\frac12\paren{v_1 + v_2}, 1 \frac12 + 2 \frac12} = \paren{\frac12\paren{v_1 + v_2}, 2 \frac12 + 1 \frac12} = g(\tilde x_2)
    \end{align*}
    If there is a $\sigma \in S$ and indices $1 \leq k_1, k_2 \leq d$ and vertices $v, v' \in V$ such that $R^\sigma(v,k_1) < R^\sigma(v,k_2)$ and $R^\sigma(v',k_1) < R^\sigma(v',k_2)$ then $g$ is non-injective. This follows from Equation \ref{eqn:g-on-vertices} and the intermediate value theorem. Indeed if $\cM$ is connected then $g(\cM)$ always has self intersections. If not, otherwise the maximal (in the final coordinate) components of $g(\cM)$ would be both open and closed, which violates connectedness of $\cM$.
    
    Now we construct the embedding $h$. From Whitney's embedding theorem and compactness there is an embedding of $\cM$ into $[0,1]^{2n}$. Let $f: \cM \to [0,1]^{2n}$ be that embedding, Further let $\epsilon > 0$ be such that $\restr{g}{\cup_{x \in X}B_\epsilon(x)}$ is injective. The existence of such an $\epsilon$ follows from continuity of $g$ and that $\restr gX$ is injective. Let $\psi_{x,\epsilon} \colon \cM \to \R$ be a bump function such that $\psi(x) = 1$, $\psi$ is smooth, and $\psi = 0$ outside of $B_\epsilon(x)$. Then we define
    \begin{align}
        \Psi(y) &= \sum_{x \in X} \psi_{x,\epsilon}(y)\\
        h(y) &= \paren{g(y), \Psi(y)f(y)}.
    \end{align}
    Let $\pi_1\colon \R^{n+1} \times \R^{2n} \to \R^{n+1}$ and $\pi_2\colon \R^{n+1} \times \R^{2n} \to \R^{2n}$ be coordinate projections. On $\cup_{x \in X}B_\epsilon(x)$, $\pi_1 \circ h = g$ is injective owning to $g$'s injectivity. On $\cM \setminus \cup_{x \in X}B_\epsilon(x)$, $\pi_2 \circ h = f$ and so is injective as well. Clearly $h$ is continuous. $h$ is a continuous injection on a compact set and so is an embedding. This concludes part 3, and the entire proof overall.
\end{enumerate}
\end{proof}

\subsection{Proof of Theorem \ref{thm:approximation-of-locdiff}}
\label{sec:proof-of-approximation-of-locdiff}

Before we prove Theorem \ref{thm:approximation-of-locdiff}, we first prove a helper lemma. This lemma shows that given a triangulation $K$ we can always find a finer triangulation $K'$ with vertices at an arbitrary set of points $Y \subset \abs{K'}$.

\begin{lemma}[Vertexed Triangulations]
    \label{lem:vertexed-triangulations}
    Let $K$ be a $n$ triangulation, and $Y \subset \abs{K}$ a finite set of points. Then there is a triangulation $K' = (V',S')$ of $K$ so that $Y \subset V'$.
\end{lemma}
\begin{proof}
    We proceed by induction on the size of $Y$. The base case when $Y = \emptyset$ is solved by letting $K' = K$. Now the inductive step. Let $y \in Y$, so that $Y = \set{y} \cup \paren{Y \setminus \set{y}}$. If $y$ is already a vertex of $K'$ then we are done. Otherwise, there is some $\sigma \in S$ so that $y \in \sigma^o$ where $\sigma^o$ denotes the relative interior of $\sigma$. Then, we can subdivide (like a barycentric subdivision but at $y$ instead of the barycenter) $\sigma$ at $y$ into $n+1$ simplices. Call these simplices $\Sigma_y$. We form a triangulation $K'$ by removing $\sigma$ and replacing it with this subdivision. For each coface $\tau \in S$ of $\sigma$, replace $\sigma$ in $\tau$ with one element each of $\Sigma_y$. A long but straight forward calculation shows that after this coface replacement is legitimate, this yields a simplicial complex $\tilde K$ with vertex at $y$. Now apply the inductive hypothesis to $\tilde K$ and the points $Y \setminus \set{y}$.
\end{proof}

Now we prove Theorem \ref{thm:approximation-of-locdiff}.

\begin{proof}
    The manifold $\cM_2$ is smooth, and so triangulable. Let $K_2 = (V_2,S_2)$ be a triangulation of $\cM_2$ with diffeomorphism $h_2\colon K_2\to \cM_2$.
    
    We apply Lemma \ref{lem:vertexed-triangulations} to the triangulation $K_2 = (V_2,S_2)$ where the set of points that we want to add to the vertex of the triangulation is $h^{-1}_2(Y)$. Note that $Y$ is finite, and $Y \subset \cM_2$, and so $h^{-1}_2(Y)\subset \abs{K_2}$, and so the lemma applies. Thus, there is a triangulation $K' = (V',S')$ of $\cM_2$ with vertices that include $h^{-1}_2(Y)$. Now we apply Theorem \ref{thm:covering-map-decomp}, and obtain projection $p \colon C \to \abs{K_2}$ and $h\in \emb^1(\cM_1,C)$ so that for $C = \abs{K_2}\times [0,d]\times [0,1]^n$ we have that $h_2^{-1}\circ g = p \circ h$. Moreover, we have that for each $y \in h^{-1}_2(Y) \subset V'$, $p^{-1}(v)\cap h(\cM_1) = \set{(v,i,\bszero)}$ for $i = 1,\dots,d$.
    
    Applying $h_2$ to both sides, we get that $g = h_2\circ p \circ h$, and for every $y \in Y$
    \begin{align*}
        g^{-1}(y) = h^{-1}\circ p^{-1} \circ h^{-1}_2(y) = h^{-1}\paren{\set{v,i,\bszero}_{i = 1,\dots,d}}
    \end{align*}
    where $v = h_2^{-1}(y)$.
    
    To form the sequence of $f_i$ and $X_i$ that satisfy Eqn.s 
        \ref{eqn:thm:approximation-of-locdiff:x-i-def} and
        \ref{eqn:thm:approximation-of-locdiff:unif-conv-of-inverses}, choose a sequence $\infseq E1i$ and $\infseq T1i$ so that $E_i$ and $T_i$ are, respectively, bistable uniform approximators to the embedding $h$ and diffeomorphism $h_2$. If we define $f_i \coloneqq T_i \circ p \circ E_i$ then by compactness $\infseq f1i$ is a bistable uniform approximator. for $g$. Further, by inverse Lipschitz-ness of $g$ and $f_i$, we have for all $y \in Y$ that 
    \begin{align}
        \sup_{y \in Y} \norm{h^{-1}\circ \paren{h^{-1}_2(y),i,\bszero} - E_i^{-1}\paren{T_i^{-1}(y),i,\bszero}_{\R^{m_1}}} \leq \epsilon_i
    \end{align}
    where $\epsilon_{i} \to 0$ as $i \to \infty$. This proves Eqn. \ref{eqn:thm:approximation-of-locdiff:unif-conv-of-inverses}.
\end{proof}

\subsection{Proof of Corollary \ref{cor:approx-to-maps-betwen-non-local-homeo-manifs}}
\label{sec:proof:cor:approx-to-maps-betwen-non-local-homeo-manifs}

In this subsection we present the proof of Corollary \ref{cor:approx-to-maps-betwen-non-local-homeo-manifs}.

\begin{proof}
    The proof of this follows from $\overline{\emb^{1}(\cM_1,\R^{m_2})} \subset \lochom(\cM_1,\cM_2)$. If $\infseq fn1$ was a uniform approximating sequence for $g$, then $g$ would necessarily be a local homeomorphism between $\cM_1$ and $\cM_2$. This is a contradiction, hence no such sequence exists.
\end{proof}

\subsection{Proof of Lemma \ref{lem:symm-and-quotient-manifs}}
\label{sec:proof:lem:symm-and-quotient-manifs}

In this section we present the proof of Lemma \ref{lem:symm-and-quotient-manifs}

\begin{proof}
    First, we prove that $\cM_1/\Sigma$ is a manifold.  Because $\Sigma$ is finite, $\pi_{\Sigma}$ is a proper and continuous. It is free because orbits are all the same size, hence $\cM_1 / \Sigma$ is a manifold from \cite[Theorem 21.10]{lee2013smooth}.
    \begin{enumerate}
        \item $\restr{f}{\cM_1 / \Sigma}$ is the unique map satisfying Eqn. \ref{eqn:symmetrization-identity}. This is shown by passing to the quotient as in \cite[Theorem A.30]{lee2013smooth}.
        \item Clearly $\pi_{\Sigma}$ is a covering space of $\cM_1 / \Sigma$ by $\cM_2$. From $f$ and $\pi_\Sigma$ being $d$-to-one we obtain that $f$ is one-to-one, and so it is a diffeomorphism, so $f_{\cM_1 / \Sigma} \circ \pi_\Sigma$ is a smooth covering map, and so is a local diffeomorphism
    \end{enumerate}
\end{proof}

\subsection{Proof of Corollary \ref{cor:recov-of-group-action}}
\label{sec:proof:cor:recov-of-group-action}

In this section we present the proof of Corollary \ref{cor:recov-of-group-action}.

\begin{proof}
    This follows from Eqn. \ref{eqn:thm:approximation-of-locdiff:unif-conv-of-inverses}. By uniform inverse lipschitzness of $E_i$, we have that $X_i$ converges $f_i^{-1}(Y)$, and so Eqn. \ref{eqn:thm:approximation-of-locdiff:unif-conv-of-inverses} becomes that $d_H(f_i^{-1}(Y),g^{-1}(Y)) \to 0$. From inverse Lipschitz-ness of $g$ and $f$ we have that, for $i$ large enough, that $f_i^{-1}(Y)$ and $g^{-1}(Y)$ have the same number of points, hence $d_H(f_i^{-1}(y),g^{-1}(y))$ for each $y \in Y$.
\end{proof}

\subsection{Proof of Lemma \ref{lem:local-dif-surfaces}}
\label{sec:proof:lem:local-dif-surfaces}

In this section we present the proof of Lemma \ref{lem:local-dif-surfaces}.

\begin{proof}

Let $m = k\paren{n-1} + 1$, where $n,m$ and $k$ are positive integers, $k\ge 2$.
In this section we explicitly construct a covering of a genus $n$ surface $S_n$ with a genus $m$ surface $S_m$. We consider coordinates $(x,y,s,z,t)$ in $\R^5$. Let $P_4:\R^5\to \R^4, P_2:\R^5\to \R^3, P_2:\R^5\to \R^2, Q_4:\R^5\to \R^3, R_2:\R^4\to \R^2$ be the projections where
\begin{align*}
    P_2(x,y,s,z,t)&=(x,y)\\
    P_3(x,y,s,z,t)&=(x,y,s)\\
    P_4(x,y,s,z,t)&=(x,y,s,z)\\
    Q_3(x,y,s,z,t)&=(x,y,z)\\
    Q_4(x,y,s,z,t)&=(x,y,z,t)\\
    R_2(x,y,z,t)&=(x,y).
\end{align*}

Consider curve $\gamma:[-2(k-1)\pi,2(k-1)\pi]\to \R^5$ given by
\begin{align}
    \gamma(r)=(\cos(r),\sin(r),0,\phi(r),\psi(r))
\end{align}
where $\phi,\psi\in C^\infty(\R)$ are such that
\begin{align*}
 &    \psi(r)\geq 0, \quad \text{ for all $r$, and}\\
   & \psi(r)>0, \quad \text{ if and only if $|r-(-2j+1)\pi|<1$ for some $j=1,2,\dots,k-1$}
   \end{align*}
   and
   \begin{align*}
   & \phi(-r)=\phi(r,) \\
    &\phi(r)=1, \quad  \text{for $r\in [0,\pi/2]$,}\\
    &\phi(r)=j+1, \quad  \text{for $r\in [(2j-\frac 12)\pi,(2j+\frac 12)\pi]$, $j=1,2,3,\dots,k-2$},\\
    &\phi(r)=k, \quad  \text{for $r\in [((k-1)-\frac 12)\pi,2(k-1)\pi]$},\\
    &\phi((2j-1)\pi)=j-\frac 12 ,\quad j=1,2,\dots,k-1&\\
    &r\to \phi(r), \quad  \text{is strictly increasing on the intervals $r\in [(2j+\frac 12)\pi,(2j+\frac 32)\pi]$, $j=0,1,2,\dots,k-2$}.
\end{align*}
Then
\begin{align}
    P_4(\gamma(-r))=P_4(\gamma(-s))\hbox{ if and only if $r= s =j\pi$ for some $j=1,2,\dots,k-1$}.
\end{align}
and $\gamma([-2(k-1)\pi,2(k-1)\pi])$ is a smooth closed curve in $\R^5$ which does not intersect itself.

Let $\mu(r)=Q_4(\gamma(r))$  be a smooth curve in $\R^4$ that has no self-intersections. 
Then $\alpha(r)=R_2(\mu(r))$ is such that the set $\alpha([0,2\pi])=\alpha([-2(k-1)\pi,2(k-1)\pi])\subset \R^2$ is the unit circle in $\R^2$ and the projection
\begin{align}
    R_2:\mu([-2(k-1)\pi,2(k-1)\pi])\to \alpha([0,2\pi])
\end{align}  
is a $k$-to-1 covering map.

We will first consider a 2-dimensional torus $\Sigma_0$ in $\R^5$ whose ``central curve'' is the path $\gamma$. In other words, a tube with central axis $\gamma$. To this end, let 
\begin{align}
    v(r)=\frac 1{|\partial_r P_2(\gamma(r))|}\partial_r P_2(\gamma(r))
\end{align} 
be the unit tangent vector of the curve $r\to P_2(\gamma(r))$ in $\R^2$ and let
\begin{align}
    \nu(r)=v(r)^\perp
\end{align}
be the  unit normal  vector of the curve $r\to P_2(\gamma(r))$ in $\R^2$ which is obtained by rotating $v(r)$ $90\degree$ clockwise in $\R^2$. Let 
\begin{align}
    \tilde \nu(r)=(\nu(r),0,0,0)
\end{align}
be a vector in $\R^5$ and $\hat e=(0,0,1,0,0)$ be a unit vector in $\R^5$
pointing to the direction of the $s$-axis.

Let $H:[-2(k-1)\pi,2(k-1)\pi]\times [-\pi,\pi]\to \R^5$ be the function
\begin{align}
    H(r,\theta)=\gamma(r)+\cos(\theta)\nu(r)+\sin(\theta)\hat e.
\end{align} 
We define 
\begin{align}
    \Sigma_0=H([-2(k-1)\pi,2(k-1)\pi]\times [-\pi,\pi])
\end{align}
to be a 2-dimensional surface in $\R^5$. The surface $\Sigma_0$ is a 2-dimensional torus in $\R^5$. It has the property that the surface $\Sigma_1=P_3(\Sigma_0)$ is a 2-dimensional torus in $\R^3$ and the projection
\begin{align}
    P_3:\Sigma_0\to \Sigma_1
\end{align}
is a 2-to-1 covering map. Observe that the points $q_j=H(2j\pi,0)=(1,0,0,j+1,0)$, $j=0,1,2,\dots,k-1$ have neighborhoods $U_j$ in $\R^5$ such that $\Sigma_0\cap U_j$ is a subset of $\R^3\times \{(j,0)\}$
and the maps
\begin{align}
    \rho_{i,j}:(x,y,s,z,t)\to (x,y,s,z+(j-i),t),\quad i,j\in \{0,1,2,\dots,k-1\},\ i<j
\end{align} 
are bijections
\begin{align}
    \rho:\Sigma_0\cap U_i\to \Sigma_0\cap U_j.
\end{align}
Now, we add $n$ handlebodies to $\Sigma_0$. We modify $\Sigma_0$ in the sets $U_j$ by smoothly gluing to $ \Sigma_0\cap U_j$ a collection of $(n-1)$ 2-dimensional toruses $T_{j,p}\subset U_j\cap \R^3\times \{(j,0)\}$,
$p=1,2,\dots,n-1$ so that we obtain a $C^\infty$-smooth surface 
\begin{align}
    S_m=\Sigma_0\# (T_{1,1} \# T_{1,2} \# \dots  \# T_{1,n-1})\# (T_{2,1} \# T_{2,2} \# \dots  \# T_{2,n-1})
\#\dots \# (T_{j,1} \# T_{j,2} \# \dots  \# T_{j,n-1})
\end{align}
that has genus $m = k\paren{n-1} + 1$ and moreover, $S_m\cap U_i\subset  \R^3\times \{(i+1,0)\}$ and the maps
\begin{align}
    \rho_{i,j}:S_m\cap U_i\to S_m\cap U_j
\end{align}
are bijection for all $i<j$. Then $S_n=P_3(S_m)$ is a smooth surface in $\R^3$ that has genus $n$ and the projection
\begin{align}
    P_3:S_m\to S_n
\end{align}
is a $k$-to-1 covering map. This show that in $\R^5$ there is a  surface $S_m$ with genus $m$ that is mapped in the projection $P_3$ to a surface $S_n$ in $\R^3$ with genus $n$ ,and for these surfaces $P_3$ is a $k$-to-1 covering map.
 
\end{proof}

\begin{figure}
    \centering
    \includegraphics[width=0.5\textwidth]{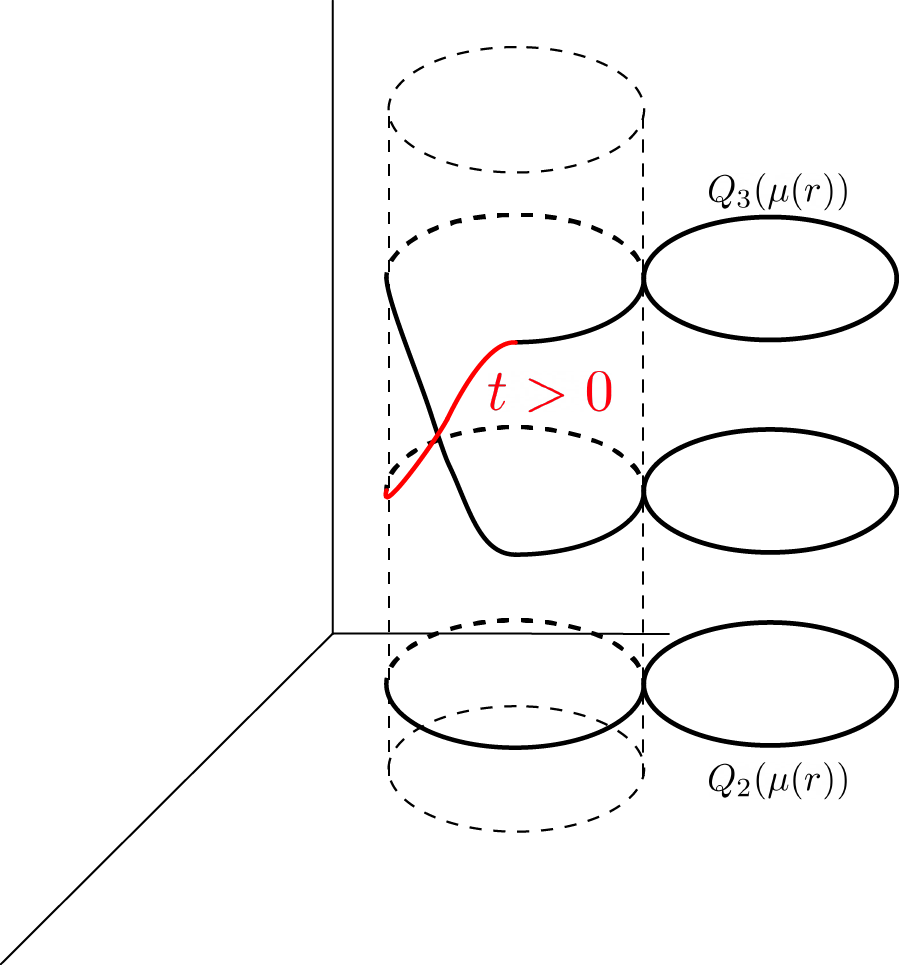}
    \caption{Explanation of the proof of Lemma \ref{lem:local-dif-surfaces} in the case when $m=3$, $n=2$ and $k=2$. Curve $Q_3(\mu(r))$ in $\R^3$. In fact, the figure can be considered also as a 4-dimensional curve $Q_4(\mu(r))$ when the red colour shows where $t>0$ and the black colour where $t=0$. Note that projection $R_2:Q_4(\mu([-2\pi,2\pi]))\to Q_2(\mu([-2\pi,2\pi]))$  is 2-to-1 covering map. With dotted lines in the figure we have also indicated the  projected  core curves of the toruses which we glue to $\Sigma_0$ to obtain a surface $S_3$ and the projected core curve of the torus that is glued in $\Sigma_1$ to obtain a surface $S_2$. }
\end{figure}

\subsection{Proof of Theorem \ref{thm:covering-maps-as-projs}}
\label{sec:proof-of-covering-maps-as-projs}

\begin{proof}
  We can assume without loss of generality that ${\tilde m_1}\ge 3m_1+1$. Otherwise, we may replace $\tilde \cM_1$ by $J_0(\tilde \cM_1)$ where $J_0\colon\R^{\tilde m_1}\to \R^{3 m_1}$ where $J_0(x)\coloneqq(x,\set{0}^{3m_1+1 - \tilde m_1})$.

    Let us now return to proof assuming that ${\tilde m_1}\ge 3m_1+1$.
    Let $J_1\colon\R^{m_1}\to \R^{m_1}\times \{0\}^{\tilde m_1-{m_1}}\subset \R^{\tilde m_1}$ be the linear injective map $J_1(x)\coloneqq(x,0)$. Then $J_1|_{\cM_1}\colon \cM_1\to \widehat  \cM_1\eqqcolon J_1(\cM_1)$ is a diffeomorphism.
    Also, let $E_1: \R^{m_1}\times \Rea^{\tilde m_1-{m_1}}\to 
     \R^{m_1}$ be the projection to the first coordinate.
    
    As $\cM_1$ and $\tilde \cM_1$ are diffeomorphic,
    there is  a diffemorphism  $g_1\colon \cM_1\to \tilde \cM_1$.
    Then, $\tilde g_1=g_1\circ E_1\colon \widehat\cM_1\to \tilde \cM_1$ is a diffeomorphism, too.
    We recall that $\cM_1\subset \Rea^{m_1}$ and $\tilde \cM_1\subset \Rea^{\tilde m_1}$.
    Then, as $\tilde m_1\ge 3m_1+1$, we can apply \cite[Lemma 7.6]{madsen1997calculus} and \cite[Thm.\ 3.8]{puthawala2022universal} to the embedding
    $g_1\colon \cM_1\to \Rea^{\tilde m_1}$
    and see that there is an injective linear map $L_1:\Rea^{m_1}\to \Rea^{\widetilde m_1}$ and a diffeomorphims $D_1:\Rea^{\widetilde m_1}\to \Rea^{\widetilde m_1}$ such that 
    $$
    g_1=D_1\circ L_1|_{\widehat  \cM_1},
    $$
    that is, the embedding $g_1$ is an extendable
    embedding (see \cite[Def.\ 3.7]{puthawala2022universal}.
    By basic results of linear algebra, there is a linear bijection $B_1:\R^{m_1}\times \Rea^{\widetilde m_1-{m_1}}=\Rea^{\widetilde m_1}\to \Rea^{\widetilde m_1}$ such that for $x\in \Rea^{m_1}$
    we have $B_1(x,0)=L_1(x)$. This
    implies that $\tilde g_1:\tilde \cM_1\to\Rea^{\widetilde m_1} $ extends to a diffeomorphism $T_1=D_1\circ B_1\colon\R^{\widetilde m_1}\to \R^{\widetilde m_1}$, so that $T_1|_{\widehat  \cM_1}=\tilde g_1$. Observe that $T_1\circ J_3|_{\cM_1}=g_3$
    is a diffeomorphims from $\cM_1$ to $\tilde \cM_1$. We recall that by assumption, the projection
    $p:\tilde \cM_2\to \tilde \cM_1$ is a covering map.
        Then the restriction of the map $p\circ T_3\circ J_3\colon\R^{m_1}\to \R^{\widetilde m_2}$ on ${\cM_1}$,  that we denote by $p\circ T_3\circ J_3|_{\cM_1}\colon \cM_1\to \tilde \cM_2$, is a covering map.
    
    Next, consider the submanifold $\cM_2$ in $\R^{m_2}$. Let us use the map $E_2\colon\R^{m_2}\to \R^{{m_2}+\widetilde k}$ where $\widetilde k = m_2 - k$, defined by $E_2(x)\coloneqq (x,0)$, to embed $\cM_2$ into $\R^{{m_2}+\widetilde k}$ to the manifold $\widehat \cM_2\coloneqq E_2(\cM_2)\subset \R^{\widetilde k}$. We then have that $p_2\colon \R^{{m_2}+\widetilde k}\to \R^{m_2}$, $p_2\colon(x,y)=x$ defines a diffeomorphism $p_2\colon\widehat  \cM_2\to \cM_2$. We assume that $\widetilde m_2+\widetilde k \ge 3m_2+1$.
    
    Similarly to the above, let $J_2\colon\R^{\widetilde m_2}\to \R^{\widetilde m_2}\times \{0\}^{\widetilde k-\widetilde m_2}\subset \R^{\widetilde k}$ be a linear injective map $J_2(x)\coloneqq (x,0)$. Then $J_2|_{\tilde \cM_2}:\tilde \cM_2\to \hat \cM_2\eqqcolon J_2(\tilde \cM_1)$ is a diffeomorphism.
    
    As there is a diffeomorphism $\tilde g_2\colon \tilde \cM_2\to \widehat  \cM_2$ and $k\ge 3m_2+1$, we see as above by using  \cite[Lemma 7.6]{madsen1997calculus} and \cite[Thm.\ 3.8]{puthawala2022universal} that $\tilde g_2$ extends to a diffeomorphism
     $T_2\colon \R^{k}\to \R^{k}$, so that $T_2|_{\tilde \cM_2}=\tilde g_2$. 
     Then we see that the map
    $$
    p_2\circ T_2\circ J_2\circ p\circ T_3\circ J_3 \bigg |_{\cM_1}\colon \cM_1\to \cM_2
    $$
    is a covering map.
\end{proof}
 
\end{document}